\documentclass{article}
\usepackage{kr}

\usepackage{times}
\usepackage{soul}
\usepackage{url}
\usepackage[hidelinks]{hyperref}
\usepackage[utf8]{inputenc}
\usepackage[small]{caption}
\usepackage{graphicx}
\usepackage{amsmath}
\usepackage{amsthm}
\usepackage{booktabs}
\usepackage{algorithm}
\usepackage{algorithmic}
\usepackage{todonotes}
\usepackage{amssymb}

\newtheorem{example}{Example}
\newtheorem{theorem}{Theorem}
\newtheorem{lemma}{Lemma}
\newtheorem{observation}{Observation}
\newtheorem{property}{Property}
\newtheorem{notation}{Notation}
\newtheorem{proposition}{Proposition}
\newtheorem{definition}{Definition}

\newcommand{\agg}{\ensuremath{\operatorname{agg}}}
\newcommand{\infl}{\ensuremath{\operatorname{infl}}}
\newcommand{\core}{\ensuremath{\mathit{Core}}}

\title{Contestability in Quantitative Argumentation}
\author{%
Xiang Yin$^1$\and
Nico Potyka$^2$\and
Antonio Rago$^1$\and
Timotheus Kampik$^3$\and
Francesca Toni$^1$\\
\affiliations
$^1$Department of Computing, Imperial College London, UK\\
$^2$School of Computer Science and Informatics, Cardiff University, UK\\
$^3$Department of Computing Science, Umeå University, Sweden\\
\emails
\{xy620, ft, a.rago\}@imperial.ac.uk,
potykan@cardiff.ac.uk,
timotheus.kampik@umu.se
}

\begin{document}

\maketitle

\begin{abstract}
Contestable AI requires that AI-driven decisions align with human preferences.
While various forms of argumentation have been shown to support contestability, Edge-Weighted Quantitative Bipolar Argumentation Frameworks (EW-QBAFs) have received little attention.
In this work, we show how
EW-QBAFs can be deployed for this purpose. 
Specifically, we introduce the \emph{contestability problem} for EW-QBAFs, which asks how to modify \emph{edge weights} (e.g., preferences)
to achieve a desired \emph{strength} for a specific argument of interest (i.e., a \emph{topic argument}).
To address this 
problem, we propose  \emph{gradient-based relation attribution explanations (G-RAEs)}, which quantify the sensitivity of the topic argument's 
strength to changes in individual edge weights, thus providing interpretable guidance for weight adjustments towards contestability.
Building on G-RAEs, we develop an iterative algorithm that progressively adjusts the edge weights to attain the desired 
strength.
We evaluate our approach experimentally 
on synthetic EW-QBAFs that simulate the structural characteristics of personalised recommender systems and multi-layer perceptrons,
and demonstrate that it can solve the problem effectively.
\end{abstract}

\section{Introduction}
Contestable AI~\cite{lyons2021conceptualising,alfrink2023contestable,constable_AI_need_AFs} 
studies AI systems that allow users to challenge their outputs. This is particularly relevant when these outputs deviate from 
expected or desirable results. 
Such deviations may arise from model errors or misalignments with 
human preferences.
To enable contestable AI, recent work~\cite{constable_AI_need_AFs} advocates  computational argumentation~(see \cite{
argumentation-survey} for an overview) as a promising paradigm, 
due to its inherent 
abilities to support 
conflict resolution, explainability and interactivity (e.g., as in \cite{cocarascu2019extracting,freedman2025argumentative,Rago_23,russo2023causal}), 
given that these properties have been acknowledged as 
important for contestable AI~\cite{lyons2021conceptualising,alfrink2023contestable,almada2019human}.

Among various forms of computational argumentation, Edge-Weighted Quantitative Bipolar Argumentation Frameworks (EW-QBAFs)~\cite{mossakowski2018modular} 
allow to naturally model reasoning over conflicting and supporting information in a quantitative way. An EW-QBAF typically consists of four components: a set of \emph{arguments}, relations (of \emph{attack} and \emph{support}), \emph{base scores} of arguments, and \emph{edge weights} of relations. The \emph{strengths} of arguments are evaluated using EW-QBAF semantics, depending on the base scores, edge weights and the strengths of their attackers and supporters. EW-QBAFs can be applied in various domains. For example, some neural networks can be understood as EW-QBAFs
, with neurons as arguments, biases as base scores, connections as relations, and connection weights as edge weights~\cite{potyka2021interpreting}. 
(Restrictions of) EW-QBAFs
can also naturally model personalised recommender systems (PRSs) \cite{cocarascu2019extracting,battaglia2024integrating,Rago_25}. For example, Figure~\ref{fig_qbaf} shows a hierarchical movie recommender system based on EW-QBAFs: the top-level argument reflects the overall evaluation of a movie, influenced by criteria arguments such as acting and writing, which are further influenced by sub-criteria arguments like specific actors. Base scores may be extracted from movie reviews, while edge weights may capture individual user preferences.

\begin{figure}[t]
    \centering
    \includegraphics[width=1.0\linewidth]{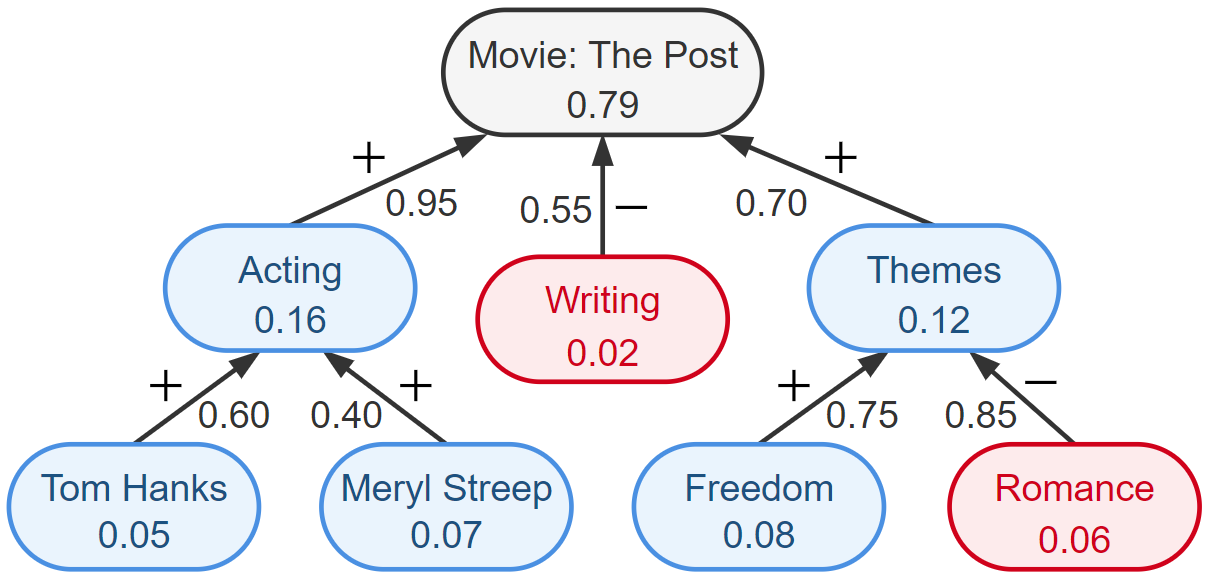}
    \caption{Example of EW-QBAF-based movie recommendation (adapted from \protect\cite{cocarascu2019extracting}). Blue and red nodes represent supporters and attackers, respectively. Edges labelled + and - indicate support and attack relations, respectively. Numeric values on the nodes and edges indicate their base scores and weights, respectively.}
    \label{fig_qbaf}
\end{figure}

Despite the inherent interpretability and broad applicability of EW-QBAFs, little attention has been given to their contestability, 
specifically systematically adjusting the edge weights to achieve a desired strength for a particular argument of interest (the \emph{topic argument}). This problem is crucial for contesting, refining, and personalising the outcome of EW-QBAFs. For example, 
in EW-QBAF-based PRSs, edge weights can be adapted to better reflect individual human preferences.

We study the following \emph{contestability problem}: \textbf{given a topic argument $\alpha$ and a desired strength $s$, how can the edge weights be modified such that $\alpha$'s strength is $s$?} 
To facilitate its solution, we introduce a notion of \emph{gradient-based relation attribution explanations (G-RAEs)} for EW-QBAFs. Intuitively, G-RAEs quantify the sensitivity of the topic argument's strength to changes in individual edge weights. They thus can serve as interpretable indicators, suggesting how much an edge weight should be increased or decreased. Based on G-RAEs, we design an iterative algorithm that progressively adjusts the edge weights to reach the desired strength.
Experimental results on synthetic EW-QBAFs, designed to simulate multi-layer perceptron (MLP) and PRS structures, demonstrate that our heuristic algorithm is \emph{effective} in attaining the desired strength, 
runs in polynomial time,
and 
scales well
to large and dense EW-QBAFs.

In summary, our main contributions are as follows:
\begin{itemize}
\item We introduce new properties for EW-QBAFs that explain the effect of edge weight changes (Section~\ref{sec_EWQBAF_properties}). 
\item We formally define the contestability problem for EW-QBAFs (Section~\ref{sec_contestability}).
\item We formally define G-RAEs and study their satisfaction of existing properties (Section~\ref{sec_grae}).
\item We develop an algorithm based on G-RAEs to solve the contestability problem (Section~\ref{sec_algo}).
\item We empirically show the effectiveness and scalability of our algorithm (Section~\ref{sec_eval}).
\end{itemize}

All technical proofs and the implementation code are provided in the supplementary material (SM).

\section{Related Work}
\label{sec_related_work}

Various proposals for designing contestable AI systems have been made in
the literature \cite{hirsch2017designing,almada2019human,alfrink2023contestable,russo2023causal}.
A recent position paper \cite{constable_AI_need_AFs} classified contestable AI solutions along three dimensions: 
the contested entities, the contesting entities, and the contestation method.
In our work, the contested entities are the topic arguments' strengths; the contesting entities are humans with different outlooks (e.g., users of PRSs or domain experts using MLPs); and the contesting method is based on the novel notion of G-RAEs. 
Argumentation naturally supports contestability.
For example, \cite{Shao_20} proposed structured argumentation-based learning methods that leverage the expressive power of universal function approximators, such as neural networks, while preserving a broad range of tractable inference routines.
These methods can provide argumentative explanations for the decision-making process, which can in turn be contested to facilitate model revision, e.g., for a correct output based on incorrect reasoning.
Similar to our work, \cite{freedman2025argumentative} employed QBAFs to enhance explainability and contestability in the context of claim verification for large language models (LLMs). In this work, the strength of a topic argument can be contested either through base score modification, or by introducing additional arguments and relations.
Our contestability problem is also  related to
the counterfactual problem introduced in \cite{CE_QArg}.
The counterfactual problem studies how base scores can be changed to change a decision (e.g., revising a credit score in loan approval systems), 
while our contestability problem focuses on modifying edge weights. The former can be seen as adjusting the apriori belief in an 
argument, whereas the latter corresponds to adjusting the relevance of an argument for a discussion.

In a wider sense, our work is related to attribution explanations in machine learning \cite{baehrens2010explain,ribeiro2016should,lundberg2017unified},
as well as
Argument Attribution Explanations (AAEs)~\cite{vcyras2022dispute,AAE_ECAI,kampik2024contribution} and Relation Attribution Explanations (RAEs)~\cite{amgoud2017measuring,RAE_IJCAI}
in argumentation.
AAEs and RAEs aim to quantify the influence of individual arguments and relations, respectively, on a topic argument.
For example, gradient-based AAEs assess this influence by computing the gradient of the topic argument's strength with respect to the influencing argument's base score, thereby capturing the former's sensitivity to the latter.
In contrast, Shapley-based RAEs quantify the influence of a relation (i.e., an attack or support) on the topic argument by leveraging Shapley values from cooperative 
game theory~\cite{shapley1951notes}.
Our G-RAEs adopt the idea of gradient-based AAEs, but focus on assessing the importance of relations, as RAEs do.

\section{Preliminaries}
\label{sec_preliminaries}
We consider EW-QBAFs similar to~\cite{mossakowski2018modular,potyka2021interpreting}.
\begin{definition}[Edge-Weighted QBAF]
\label{def_QBAF}
An \emph{Edge-Weighted QBAF}
is a quintuple $\mathcal{Q}=\left\langle\mathcal{A}, \mathcal{R}^{-}, \mathcal{R}^{+}, \tau, w  \right\rangle$ where:
\begin{itemize}
    \item $\mathcal{A}$ is a finite set of \emph{arguments};
    \item $\mathcal{R}^{-} \subseteq \mathcal{A} \times \mathcal{A}$ is a binary \emph{attack} relation;
    \item $\mathcal{R}^{+} \subseteq \mathcal{A} \times \mathcal{A}$ is a binary \emph{support} relation;
    \item $\mathcal{R}^{-} \cap \mathcal{R}^{+} = \emptyset$;
    \item $\tau: \mathcal{A} \rightarrow  [0,1]$ is a \emph{base score function};
    \item $w: \mathcal{{R}^{-}\cup{R}^{+}} \rightarrow [0,1]$ is an \emph{edge weight function}.
\end{itemize}
\end{definition}

\begin{notation}
    Let $\mathcal{R} = \mathcal{R}^{-} \cup \mathcal{R}^{+}$.
\end{notation}
EW-QBAF semantics assign a strength to every argument.
\begin{definition}[Edge-Weighted Gradual Semantics]
\label{def_semantics}
An \emph{edge-weighted gradual semantics} is a function 
$\sigma: \mathcal{A} \rightarrow  [0,1] \cup \{\bot\}$. We call
$\sigma(\alpha)$ the \emph{strength} of $\alpha$ and say that it is \emph{undefined} only if
$\sigma(\alpha) = \bot$.
\end{definition}
In the remainder, unless specified otherwise, we will assume 
an EW-QBAF $\mathcal{Q}=\left\langle\mathcal{A}, \mathcal{R}^{-}, \mathcal{R}^{+}, \tau, w \right\rangle$.
\begin{notation}
For any \( \alpha, \beta \in \mathcal{A} \), we define the sets of incoming attacks, supports, and edges of \( \alpha \) as follows:
\begin{itemize}
    \item $\mathcal{R}^{-}(\alpha) = \{ (\beta, \alpha) \mid (\beta, \alpha) \in \mathcal{R}^{-} \},$
    \item $\mathcal{R}^{+}(\alpha) = \{ (\beta, \alpha) \mid (\beta, \alpha) \in \mathcal{R}^{+} \},$
    \item $\mathcal{R}(\alpha) = \mathcal{R}^{-}(\alpha) \cup \mathcal{R}^{+}(\alpha).$
\end{itemize}
\end{notation}
QBAF semantics usually belong to the family of modular semantics \cite{mossakowski2018modular}. For these semantics,
strength values are computed iteratively by an update procedure that initializes the strength values with the base score
and updates them repeatedly. They are called modular because the update function can be decomposed into an \emph{aggregation function}
that aggregates the strength values of attackers and supporters, and an \emph{influence function} that adapts the base score
based on the aggregate. The aggregation function has the form $\agg(A,S)$, where $A$ and $S$ are multisets of attack and
support values. Examples include
\begin{description}
    \item[Sum:] $\agg_{\Sigma}(A,S) = \sum_{x \in S} x - \sum_{x \in A} x$,
    \item[Product:] $\agg_{\Pi}(A,S) = \prod_{x \in A} (1 - x) - \prod_{x \in S} (1 - x)$.
\end{description}
Product-aggregation is used for the DF-QuAD 
semantics \cite{rago2016discontinuity}, while sum-aggregation is used for
the restricted Euler-based (REB) semantics \cite{amgoud2016evaluation,amgoud2018evaluation}, 
the quadratic energy (QE) semantics \cite{Potyka18}
and the MLP-based semantics \cite{potyka2021interpreting}.
In the standard QBAF setting, $A$ and $S$ are initialized with the strength values of attackers and supporters.
In the EW-QBAF setting, we can use the edge-weighted strength values instead.
That is, for a (temporary) strength function $\sigma$ and an argument $\alpha$, we 
consider the multisets
\begin{align}
 &A^\sigma_\alpha = \{ \sigma(\beta) \cdot w((\beta, \alpha)) \mid \beta \in \mathcal{R}^{-}(\alpha)\}, \\
 &S^\sigma_\alpha = \{ \sigma(\beta) \cdot w((\beta, \alpha)) \mid \beta \in \mathcal{R}^{+}(\alpha)\}.
\end{align}
The aggregation
can then be fed into the influence function as usual.
Influence functions have the form $\infl(B,A)$, where $B$ is a base score and $A$ is a real number
(the aggregate computed by the aggregation function). Intuitively, the influence function will increase
(decrease) the base score if the aggregate is larger (smaller) than $0$ while respecting the bounds
$0$ and $1$ of the strength domain. For example, the influence functions of the 
 DF-QuAD \cite{rago2016discontinuity} and QE \cite{Potyka18} semantics
 have the form
 \begin{align*}    
 \infl(B,A) = B - B \cdot h(-A) + (1-B) \cdot h(A),
 \end{align*}
 where $h(x) = \max \{0, x\}$ for DF-QuAD
 and
 $h(x) = \frac{\max \{0, x\}^2}{1 + \max \{0, x\}^2}$ for QE.

The general algorithm for computing strength values under modular semantics can be described as follows:
\begin{description}
    \item[Initialization:] for all arguments $\alpha$, let $\sigma^{(0)}(\alpha) = \tau(\alpha)$.
    \item[Update:] for $i \in \mathbb{N}_0$: for all arguments $\alpha$, let $\sigma^{(i+1)}(\alpha) = \infl(\tau(\alpha), \agg(A^{\sigma^{(i)}}_\alpha,S^{\sigma^{(i)}}_\alpha))$. Stop if difference between $\sigma^{(i)}$ and $\sigma^{(i+1)}$
    is sufficiently small.
\end{description}
For acyclic QBAFs, the algorithm is equivalent to a simple forward pass with respect to a topological ordering
of the arguments, and the strength values can be computed in linear time \cite{Potyka19}. For cyclic QBAFs,
sufficient conditions for convergence exist, but they need to make strong assumptions about the indegree of arguments
or the base scores \cite{mossakowski2018modular,Potyka19}. While the conditions are not necessary, the literature
contains various examples of cyclic QBAFs where the algorithm fails to converge \cite{mossakowski2018modular,PotykaB24}.
In this case, strength values remain undefined.
However, in all known cases, convergence problems can be solved by continuizing the semantics \cite{Potyka18,Potyka19,PotykaB24}.




\begin{definition}[Paths]
For any $\alpha, \beta \in \mathcal{A}$,
we let $p_{\alpha \mapsto \beta}=\langle(\gamma_0,\gamma_1),\cdots,(\gamma_{n-1},\gamma_{n})\rangle (n \geq 1)$ denote a \emph{path} from $\alpha$ to $\beta$, 
where $\alpha=\gamma_0$, $\beta=\gamma_n$, 
$\gamma_{i} \in \mathcal{A} (1 \leq i \leq n)$
and $(\gamma_{i-1},\gamma_{i}) \in \mathcal{R}$.
\end{definition}

\begin{notation}
Let $\left| p_{\alpha \mapsto \beta} \right|$ denote the length of path $p_{\alpha \mapsto \beta}$.
Let $P_{\alpha \mapsto \beta}$ and $|P_{\alpha \mapsto \beta}|$ denote the set of all paths from $\alpha$ to $\beta$, and the number of paths in $P_{\alpha \mapsto \beta}$, respectively.
\end{notation}

\begin{definition}[Edge Types]
\label{def_edge_type}
Let $\alpha, \beta, \gamma \in \mathcal{A}$ be pairwise distinct and $(\beta, \alpha),(\beta, \gamma) \in \mathcal{R}$.
    \begin{enumerate}
        \item $(\beta, \alpha) \in \mathcal{R}$ is a \emph{direct} edge w.r.t. $\alpha$;
        \item $(\beta, \gamma) \in \mathcal{R}$ is an \emph{indirect} edge w.r.t. $\alpha$ if 
        $|P_{\gamma \mapsto \alpha}|=1$;
        \item $(\beta, \gamma) \in \mathcal{R}$ is a \emph{multifold} edge w.r.t. $\alpha$ if $|P_{\gamma \mapsto \alpha}|>1$;
        \item $(\beta, \gamma) \in \mathcal{R}$ is an \emph{independent} edge w.r.t. $\alpha$ if $|P_{\gamma \mapsto \alpha}|=0$.
    \end{enumerate}
\end{definition}



\begin{example}
We evaluate the EW-QBAF in Figure~\ref{fig_qbaf} using MLP-based semantics~\cite{potyka2021interpreting}. Since there are five arguments that have no attackers or supporters, their strengths equal their base scores:
$\sigma(Tom\ Hanks)=0.050$,
$\sigma(Meryl\ Streep)=0.070$,
$\sigma(Freedom)=0.080$,
$\sigma(Romance)=0.060$, and
$\sigma(Writing)=0.020$.
Then, by applying aggregation and influence function, we have
$\sigma(Acting)=0.168$,
$\sigma(Themes)=0.125$, and
$\sigma(Movie)=0.827$.
\end{example}

\section{Properties of EW-QBAFs}
~\label{sec_EWQBAF_properties}

Before introducing the contestability problem, we discuss some properties of EW-QBAFs that will be useful in the following.
The properties are similar to existing properties in the literature, but pertain to the
effect of edge weights, whereas existing properties only cover the effect of base scores.
We recall some definitions from \cite{PotykaB24}.
\begin{itemize}
    \item For a multiset $S$ of real numbers, we let 
$\core(S) = \{x \in S \mid x \neq 0\}$
denote the sub-multiset that contains only the
non-zero elements.
\item Let  $S, T$ be multisets of real numbers.
$S$ \emph{dominates}
$T$ if $\core(S) = \core(T) = \emptyset$
or there is a
sub-multiset $S' \subseteq \core(S)$ and a bijective function
$f: \core(T) \rightarrow S'$ such that
$x \leq f(x)$ for all
$x \in \core(T)$.
If, in addition, $|\core(S)| >  |\core(T)|$ or
there is an $x \in \core(T)$ such that $x < f(x)$, $S$ \emph{strictly dominates}
$T$.
We write $S \succeq T$ ($S \succ T$) if $S$ (strictly)
dominates $T$.
\item $S$ and $T$ are \emph{balanced} if
$\core(S) = \core(T)$.
We write $S \cong T$ in this case.
\item  An aggregation function $\agg$ satisfies balance iff
\begin{enumerate}
    \item $A \cong S$ implies $\agg(A ,S) = 0$, and
    \item $A \cong A'$ and $S \cong S'$ implies $\agg(A ,S) = \agg(A' ,S')$.
\end{enumerate}
\item An aggregation function $\agg$ satisfies \emph{Neutrality} iff
$
\agg(A ,S) =
\agg(\core(A),\core(S)).
$
\item An influence function $\infl$ satisfies \emph{Balance}
if $\infl(b, 0)\!\!=\!\!b$.
\item An aggregation function $\agg$ satisfies \emph{Monotonicity} iff
\begin{enumerate}
    \item $\agg(A,S) \leq 0$ if $A \succeq S$, and
    \item $\agg(A,S) \geq 0$ if $S \succeq A$, and
    \item $\agg(A,S) \leq \agg(A',S)$ if $A \succeq A'$,\item $\agg(A,S) \geq \agg(A,S')$ if $S \succeq S'$.
\end{enumerate}
If the inequalities can be replaced by strict inequalities
for strict domination, $\agg$ satisfies \emph{Strict Monotonicity}.
\item An influence function $\infl$ satisfies \emph{Monotonicity} iff
\begin{enumerate}
    \item $\infl(b, a) \leq b$ if $a < 0$, and
    \item $\infl(b, a) \geq b$ if $a > 0$, and
    \item $\infl(b_1, a) \leq \infl(b_2, a)$ if $b_1 < b_2$, and
    \item $\infl(b, a_1) \leq \infl(b, a_2)$ if $a_1 < a_2$.
\end{enumerate}
If the inequalities can be replaced by strict inequalities,
 when excluding $b=0$ for the first and $b=1$ for the second item,
 $\infl$ satisfies \emph{strict monotonicity}.
\end{itemize}

Our first property is a variant of the \emph{neutrality} property~\cite{amgoud2018evaluation} which states that arguments with strength $0$
have no effect. Our \emph{edge-neutrality} property demands similarly that an edge with weight $0$ has no impact. 
\begin{definition}[Edge-Neutrality]
A semantics $\sigma$ satisfies \emph{edge-neutrality} iff, for any $\mathcal{Q}$ and $\mathcal{Q}'=\langle \mathcal{A}, \mathcal{R^-}', \mathcal{R^+}', \tau, w' \rangle$ such that $\mathcal{R^-} \subseteq \mathcal{R^-}'$, $\mathcal{R^+} \subseteq \mathcal{R^+}'$
and $w'(r) = w(r)$ for all $r \in \mathcal{R}$, the following holds: for any $\alpha,\beta \in \mathcal{A}$, let $\sigma_{\mathcal{Q}'}(\alpha)$ denote the strength of $\alpha$ in $\mathcal{Q}'$, if $\mathcal{R^-}' \cup \mathcal{R^+}'=\mathcal{R^-} \cup \mathcal{R^+} \cup \{(\beta,\alpha)\}$ and 
$w'((\beta,\alpha))=0$, then $\sigma(\alpha) = \sigma_{\mathcal{Q}'}(\alpha)$.
\end{definition}
\begin{lemma}
If a modular semantics is based on an aggregation function $\agg$ that satisfies neutrality, then
the semantics satisfies edge-neutrality.
\end{lemma}
The following \emph{edge-stability} property is similar to the (base score) \emph{stability} property \cite{amgoud2018evaluation}.
\begin{definition}[Edge-Stability]
\label{def_edge_stability}
A gradual semantics $\sigma$ satisfies \emph{edge-stability} iff for any $\alpha \in \mathcal{A}$, whenever $w(r)=0$ for all $r \in \mathcal{R}(\alpha)$, then $\sigma(\alpha) = \tau(\alpha)$.
\end{definition}
\begin{lemma}
If a modular semantics is based on an aggregation function $\agg$
and an influence function $\infl$ which satisfy their associated balance properties, then
the semantics satisfies edge-stability.
\end{lemma}
The next property is a variant of the \emph{directionality} property~\cite{amgoud2016evaluation}.
\begin{definition}[Edge-Directionality]
\label{def_directionality}
A semantics $\sigma$ satisfies \emph{edge-directionality} iff, for any $\mathcal{Q}$ and $\mathcal{Q}'=\langle \mathcal{A}', \mathcal{R^-}', \mathcal{R^+}', \tau', w' \rangle$ such that $\mathcal{A}=\mathcal{A}'$, $\mathcal{R^-} \subseteq \mathcal{R^-}'$, and $\mathcal{R^+} \subseteq \mathcal{R^+}'$, the following holds: for any $\alpha,\beta,\gamma \in \mathcal{A}$, let $\sigma_{\mathcal{Q}'}(\gamma)$ denote the strength of $\gamma$ in $\mathcal{Q}'$, if $\mathcal{R^-}' \cup \mathcal{R^+}'=\mathcal{R^-} \cup \mathcal{R^+} \cup \{(\alpha,\beta)\}$ and 
$P_{\beta \mapsto \gamma} = \emptyset$, then $\sigma(\gamma) \equiv \sigma_{\mathcal{Q}'}(\gamma)$ for any $w'((\alpha,\beta)) \in [0,1]$.
\end{definition}
\begin{lemma}
Every modular semantics satisfies edge-directionality.    
\end{lemma}
We next consider two variants of \emph{monotonicity} properties proposed in the literature \cite{CE_QArg,baroni2018many}. 
\begin{definition}[Monotonicity]
\label{def_mono}
A gradual semantics $\sigma$ satisfies \emph{monotonicity} iff, for any $\mathcal{Q}$ and $\mathcal{Q}'=\langle \mathcal{A}, \mathcal{R^-}, \mathcal{R^+}, \tau'\rangle$, 
the following holds: 
for any $\alpha,\beta \in \mathcal{A}$ $(\alpha \neq \beta)$ such that $(\beta,\alpha)\in\mathcal{R}$ and $|P_{\beta \mapsto \alpha}|=1$, let $\sigma_{\mathcal{Q}'}(\alpha)$ denote the strength of $\alpha$ in $\mathcal{Q}'$,
for any $\tau':\mathcal{A} \rightarrow [0,1]$:
    \begin{enumerate}
        
        \item 
        
        If $(\beta, \alpha) \in \mathcal{R}^{-}$, $\tau(\beta) \leq \tau'(\beta)$ and $\tau(\gamma) = \tau'(\gamma)$ for all $\gamma \in \mathcal{A} \setminus \{\beta\}$, then $\sigma(\alpha) \geq \sigma_{\mathcal{Q}'}(\alpha)$;
        \item 
        
        If $(\beta, \alpha) \in \mathcal{R}^{+}$, $\tau(\beta) \leq \tau'(\beta)$ and $\tau(\gamma) = \tau'(\gamma)$ for all $\gamma \in \mathcal{A} \setminus \{\beta\}$, then $\sigma(\alpha) \leq \sigma_{\mathcal{Q}'}(\alpha)$.
    \end{enumerate}
\end{definition}
\begin{definition}[Edge-Monotonicity]
\label{def_edge_monotonicity}
A gradual semantics $\sigma$ satisfies \emph{edge-monotonicity} iff, for any $\mathcal{Q}$ and $\mathcal{Q}'=\langle \mathcal{A}, \mathcal{R^-}, \mathcal{R^+}, \tau, w' \rangle$, the following holds:
for any $\alpha \in \mathcal{A}$, let $\sigma_{\mathcal{Q}'}(\alpha)$ denote the strength of $\alpha$ in $\mathcal{Q}'$, for any $r \in \mathcal{R}$ and $w':\mathcal{R} \rightarrow [0,1]$:
    \begin{enumerate}
        \item If $r \in \mathcal{R}^{-}(\alpha)$, $w(r) \leq w'(r)$ and $w(t) = w'(t)$ for all $t \in \mathcal{R} \setminus \{r\}$, then $\sigma(\alpha) \geq \sigma_{\mathcal{Q}'}(\alpha)$;
        \item If $r \in \mathcal{R}^{+}(\alpha)$, $w(r) \leq w'(r)$ and $w(t) = w'(t)$ for all $t \in \mathcal{R} \setminus \{r\}$, then $\sigma(\alpha) \leq \sigma_{\mathcal{Q}'}(\alpha)$.
    \end{enumerate}
\end{definition}
\begin{lemma}
If a modular semantics is based on an aggregation function $\agg$
and an influence function $\infl$ which satisfy their associated monotonicity properties, then
the semantics satisfies monotonicity and edge-monotonicity for the class of acyclic EW-QBAFs.    
\end{lemma}
All commonly considered aggregation and influence functions satisfy their associated balance,
monotonicity and neutrality properties \cite[Lemma 10, 14]{PotykaB24}.
Thus, our previous lemmas imply the following result.
\begin{proposition}
\label{prop_mono_sta_dir}
Edge-Weighted QE, REB, DF-QuAD and MLP-based semantics satisfy 
edge-neutrality,
edge-stability,
edge-directionality,
monotonicity and
edge-monotonicity.
\end{proposition}
Finally, later we will use the fact that the strength function under these semantics is differentiable with respect to
edge-weights when the EW-QBAF is acyclic. 
\begin{lemma}
\label{lemma_differentiability}
For acyclic EW-QBAFs, the strength function under Edge-Weighted QE, REB, DF-QuAD and MLP-based semantics
is differentiable with respect to edge-weights.
\end{lemma}

In the remainder, we will use $\sigma$ for any of these edge-weighted gradual semantics.

\section{Contestability Problem}
\label{sec_contestability}

In this section, we define and study the contestability problem for EW-QBAFs.
We will assume that the EW-QBAFs are acyclic. 
While this is a restriction, many applications such as PRSs \cite{cocarascu2019extracting,battaglia2024integrating,Rago_25}
naturally result in acyclic graphs due to their hierarchical structure.

Intuitively, the contestability problem for EW-QBAFs
is to find a modification of edge weights that yields 
a desired strength for a specified topic argument.
\begin{definition}[Contestability Problem]
\label{def_contest}
Given a \emph{topic argument} $\alpha \in \mathcal{A}$ 
and a desired strength $s$ for $\alpha$ such that $\sigma(\alpha) \neq s$ in $\mathcal Q$, 
let $\mathcal{Q}'=\langle \mathcal{A}, \mathcal{R^-}, \mathcal{R^+}, \tau, w' \rangle$ and $\sigma_{\mathcal{Q}'}(\alpha)$ denote the strength of $\alpha$ in $\mathcal{Q}'$,
the \emph{contestability problem} is to identify an edge weight function $w'$ such that $\sigma_{\mathcal{Q}'}(\alpha) = s$.
\end{definition}

As an example, consider the EW-QBAF in Figure~\ref{fig_qbaf}, where a user disagrees with the current movie rating score $\sigma(\alpha)=0.827$ and instead prefers a lower strength $s=0.3$. This scenario can be viewed as a contestability problem, in which the edge weights can be adjusted to better capture users' preferences.

Our first question is whether any desired strength for a given topic argument can be attained. 
Cocarascu et al.~\shortcite{cocarascu2019extracting}
introduced an \emph{attainability property} in the standard QBAF setting, which examines whether a desired strength for a topic argument can be attained with a certain set of attackers or supporters. Here, we adapt this notion to the EW-QBAF setting, where we focus on whether a desired strength for a topic argument can be attained by modifying the edge weights.
\begin{definition}[Attainability]
\label{def_attainability}
For any $\alpha \in \mathcal{A}$ and $s \in [0,1]$, we say that 
$s$ is \emph{attainable for} $\alpha$ iff there exists an edge weight function $w'$ such that $\sigma_{w'}(\alpha) = s$. The \emph{attainable set} of $\alpha$ is defined as
$S(\alpha) = \{s \mid s \text{ is attainable for } \alpha\}$.
\end{definition}
\begin{observation}
For any $\alpha \in \mathcal{A}$ and any desired strength $s \in [0,1]\setminus\{\sigma(\alpha)\}$ for $\alpha$, the solution to the contestability problem exists iff $s \in S(\alpha)$. 
\end{observation}

The following proposition presents a special case where the base score is always attainable. 
\begin{proposition}[Base Score Attainability]
For any $\alpha \in \mathcal{A}$, 
if $\sigma$ satisfies edge-stability,
then  $\tau(\alpha) \in S(\alpha)$.
\end{proposition}

To find the boundaries of a topic argument's attainable strength, we first define the $\max$ and $\min$ edge weight functions (illustrated in Figure~\ref{fig_max_min}). These functions represent two extreme cases: one where the topic argument receives maximal supports and no attacks, and another where it receives maximal attack and no support.
\begin{definition}[Max and Min Edge Weight Functions]
    For any \( \alpha \in \mathcal{A} \),
    the \emph{max} and \emph{min edge weight functions} 
    $w^{\alpha}_{\max}$ and $w^{\alpha}_{\min}$ (respectively) are defined as follows:
\begin{itemize}
    \item\( w^{\alpha}_{\max}(r) = 1 \) for all \( r \in \mathcal{R}^{+} \);
    \item\( w^{\alpha}_{\max}(t) = 0 \) for all \( t \in \mathcal{R}^{-} \);
    \item\( w^{\alpha}_{\min}(r) = 1 \) for all \( r \in \mathcal{R}^{-}(\alpha) \cup (\mathcal{R}^{+} \setminus \mathcal{R}^{+}(\alpha)) \);
    \item\( w^{\alpha}_{\min}(t) = 0 \) for all \( t \in \mathcal{R}^{+}(\alpha) \cup (\mathcal{R}^{-} \setminus \mathcal{R}^{-}(\alpha)) \).
\end{itemize}
\end{definition}

Intuitively, to maximise $\sigma(\alpha)$ (Figure~\ref{fig_max_min}, left), we set all supports of $\alpha$ to 1 and attacks to 0 (e.g., $w(r_{10})=w(r_{12})=1$, $w(r_{11})=0$). We similarly boost its supporters and attackers (e.g., $w(r_4)=1$, $w(r_8)=0$; $w(r_6)=w(r_7)=w(r_9)=1$, $w(r_5)=0$), as attackers may indirectly increase $\sigma(\alpha)$. Recursively applying this yields all support edges as $1$, attack edges as $0$, and the maximal $\sigma(\alpha)$.



\begin{figure}[t]
    \centering
    \includegraphics[width=1.0\linewidth]{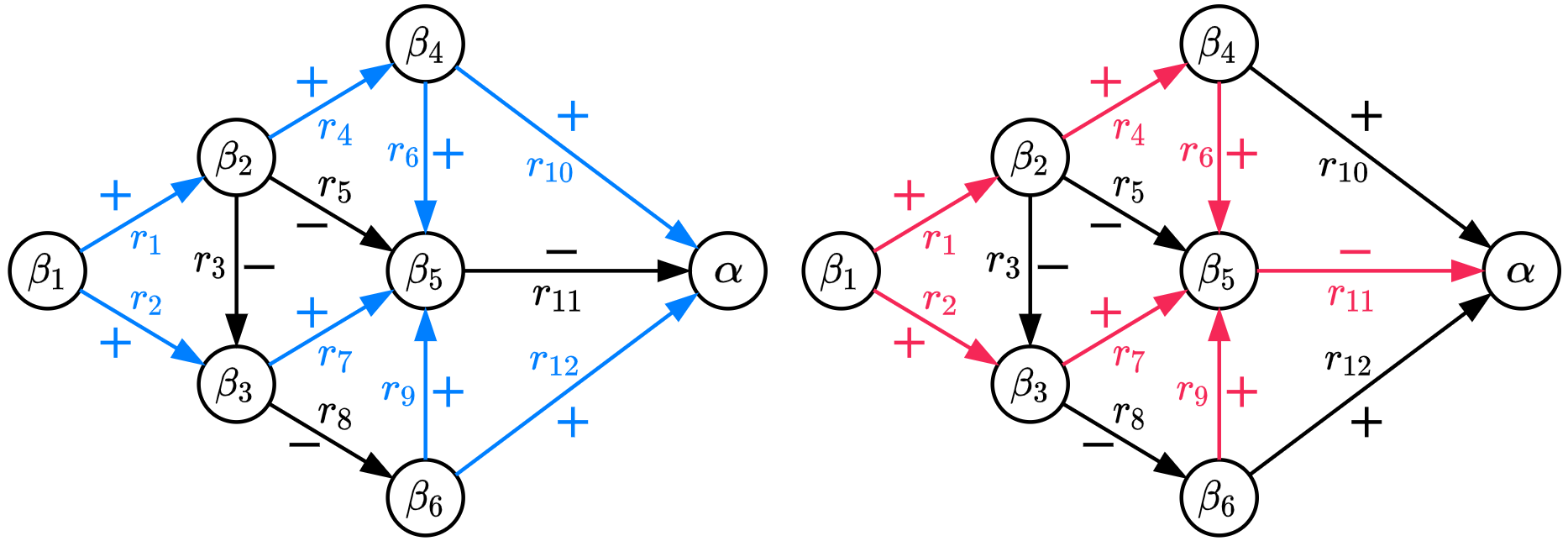}
    \caption{Illustration of the $\max$ (left) and $\min$ (right) edge weight functions. Blue or red edges are assigned a weight of 1, while black edges are assigned a weight of 0.}
    \label{fig_max_min}
\end{figure}
The following theorem states that the max and min edge weight functions determine the maximum and minimum of the attainable set.
\begin{theorem}[Maximum and Minimum of Attainability]
\label{theorem1}
For any \( \alpha \in \mathcal{A} \) and any $\sigma$ satisfying edge-monotonicity, monotonicity, and edge-neutrality, \( \sigma_{w^{\alpha}_{\max}}(\alpha) \) and \( \sigma_{w^{\alpha}_{\min}}(\alpha) \) are the maximum and minimum of \( S(\alpha) \), respectively.
\end{theorem}

While Theorem~\ref{theorem1} characterises the boundaries of the attainable set, the next natural question is whether any value between the maximum and the minimum can be attained, which is confirmed by the following theorem.
\begin{theorem}[Completeness of Attainability]
\label{theorem2}
For any \( \alpha \in \mathcal{A} \),
let $m$ and $M$ denote the minimum and maximum of $S(\alpha)$, respectively. If $\sigma$ satisfies continuity, then $S(\alpha) =[m,M]$.
\end{theorem}

In this section, we introduced the contestability problem and investigated its solvability. We will now introduce an explanation method that can guide the search for a solution.

\section{Explanations and Properties}
\label{sec_grae}
In this section, we propose a novel notion of a \emph{gradient-based relation attribution explanation (G-RAE)} and study its properties adapted from the literature.

\subsection{G-RAEs}
In order to attain a desired strength of a topic argument in acyclic EW-QBAFs, we combine the ideas of gradient-based AAEs and RAEs, and propose a novel notion of \emph{G-RAEs}, which capture the sensitivity of the strength of a topic argument with respect to the changes of individual edge weights.
\begin{definition}[Gradient-based Relation Attribution Explanations (G-RAEs)]
\label{def_grae}
Let $r \in \mathcal{R}$ and $\alpha \in \mathcal{A}$ be a \emph{topic argument}.
For a perturbation $\varepsilon \in [-w(r),0) \cup (0,1-w(r)]$, let $w'$ be an edge weight function such that $w'(r)=w(r)+\varepsilon$ and $w'(t) = w(t)$ for all $t \in \mathcal{R} \setminus \{r\}$.
The \emph{G-RAE} from $r$ to $\alpha$ under $\sigma$ is 
$$\nabla_{r \mapsto \alpha}^{\sigma}=\lim_{\varepsilon \to 0}\frac{\sigma_{w'}(\alpha)-\sigma(\alpha)}{\varepsilon}.$$
\end{definition}
Let us note that, by Lemma \ref{lemma_differentiability}, $\nabla_{r \mapsto \alpha}^{\sigma} \in \mathbb{R}$ is always well-defined in acyclic EW-QBAFs.  
We next distinguish attribution influence based on the sign of G-RAE.
\begin{definition}[Attribution Influence]
\label{def_attribution_influence}
We say that the \emph{attribution influence} from
$r$ to $\alpha$ is \emph{positive} if $\nabla_{r \mapsto \alpha}^{\sigma}>0$, \emph{negative} if $\nabla_{r \mapsto \alpha}^{\sigma}<0$, and \emph{neutral} if $\nabla_{r \mapsto \alpha}^{\sigma}=0$.
\end{definition}

As explanations scores, G-RAEs reveal both the direction and magnitude of each edge's influence on the topic argument, which could be practically useful in various domains such as PRS contestation or MLP debugging, where identifying and adjusting the most influential edges may directly enhance user satisfaction or improve model performance.

\begin{example}
\label{example_GRAE}
Consider the EW-QBAF in Figure~\ref{fig_qbaf} where $\sigma$ is given by MLP-based semantics. Table~\ref{tab_gradient_movie} shows their G-RAEs w.r.t. the topic argument \emph{Movie} in decreasing order.
\begin{table}[h!]
\centering
\small
\begin{tabular}{lr}
\toprule
\textbf{Edge} & \textbf{G-RAE} \\
\midrule
$(Acting, Movie)$ & $0.02408$ \\
$(Themes, Movie)$ & $0.01799$ \\
$(Meryl, Acting)$ & $0.00133$ \\
$(Tom, Acting)$ & $0.00095 $\\
$(Freedom, Themes)$ & $0.00088$ \\
$(Romance, Themes)$ & $-0.00066$ \\
$(Writing, Movie)$ & $-0.00287$ \\
\bottomrule
\end{tabular}
\caption{G-RAEs with respect to the topic argument \textit{Movie}.}
\label{tab_gradient_movie}
\end{table}
\end{example}

\subsection{Properties of G-RAEs}
We next adapt some commonly used properties of argumentative attribution explanations from \cite{AAE_ECAI,RAE_IJCAI}, and examine the satisfaction of these properties.

We start by analysing the influence of \textbf{direct}, \textbf{indirect} and \textbf{independent} edges (Definition~\ref{def_edge_type}) in the following three propositions.
Direct Influence~\cite{AAE_ECAI,RAE_IJCAI} shows that our G-RAEs correctly capture the qualitative effects of direct connectivity: a direct attack (support) always yields non-positive (non-negative, respectively) influence on the topic argument.
\begin{proposition}[Direct Influence]
\label{proposition_direct_influence}
Let $r \in \mathcal{R}$ be a direct edge w.r.t. $\alpha \in \mathcal{A}$. Then the following statements hold if $\sigma$ satisfies edge-monotonicity:

1. If $r \in \mathcal{R^{+}}$, then $\nabla_{r \mapsto \alpha}^{\sigma} \geq 0$;

2. If $r \in \mathcal{R^{-}}$, then $\nabla_{r \mapsto \alpha}^{\sigma} \leq 0$.

\end{proposition}

We next make several differentiations for indirect connectivity~\cite{Rago_23,RAE_IJCAI} because the polarity of an edge may invert along the path. For example, an attacker of an attacker actually serves as a supporter. 
\begin{proposition}[Indirect Influence]
\label{proposition_sign_correct_indirect}
Let $r$ be an indirect edge w.r.t. $\alpha$. Suppose the path sequence from $r$ to $\alpha$ is $\langle r, r_{1}, \cdots, r_{n} \rangle (n \geq 1)$. Let $\lambda= \left| \{r_{1}, \cdots, r_{n} \} \cap \mathcal{R}^{-} \right|$. Then the following statements hold if $\sigma$ satisfies both monotonicity and edge-monotonicity.

    1. If $r \in \mathcal{R^{+}}$ and $\lambda$ is odd, then $\nabla_{r \mapsto \alpha}^{\sigma} \leq 0$;
    
    2. If $r \in \mathcal{R^{+}}$ and $\lambda$ is even, then $\nabla_{r \mapsto \alpha}^{\sigma} \geq 0$;

    3. If $r \in \mathcal{R^{-}}$ and $\lambda$ is odd, then $\nabla_{r \mapsto \alpha}^{\sigma} \geq 0$;
    
    4. If $r \in \mathcal{R^{-}}$ and $\lambda$ is even, then $\nabla_{r \mapsto \alpha}^{\sigma} \leq 0$.
\end{proposition}

\begin{example}
In Example~\ref{example_GRAE}, since $(Acting,Movie)$ and $(Writing,Movie)$ are direct support and attack w.r.t. the topic argument $Movie$, their G-RAEs are positive and negative, respectively, by Proposition~\ref{proposition_direct_influence}. In addition, since the path from $Romance$ to $Movie$ contains one attack 
(i.e., an odd number of attacks), the G-RAE of $(Romance,Movie)$ is negative by Proposition~\ref{proposition_sign_correct_indirect}.
\end{example}

Irrelevance states that any edge independent of the topic argument has no influence.
\begin{proposition}[Irrelevance]
\label{proposition_orrelevance}
Let $\alpha \in \mathcal{A}$. If $r \in \mathcal{R}$ is an independent edge w.r.t.
$\alpha$ and $\sigma$ satisfies edge-directionality, then $\nabla_{r \mapsto \alpha}^{\sigma} = 0$.
\end{proposition}

In the following, we adapt two existing properties to align with the EW-QBAF setting and examine their satisfaction in the direct and indirect cases. Counterfactuality, inspired by \cite{AAE_ECAI,RAE_IJCAI}, considers how the strength of a topic argument changes when an edge is removed, which, in our setting, corresponds to setting edge weight to 0.
\begin{property}[Counterfactuality]
\label{property_rae_counterfactuality}
Let $\alpha \in \mathcal{A}$ and $r \in \mathcal{R}$.
Let $w'$ be an edge weight function such that $w'(r)=0$ and $w'(t) = w(t)$ for all $t \in \mathcal{R} \setminus \{r\}$.

1. If $\nabla_{r \mapsto \alpha}^{\sigma} < 0$, then $\sigma(\alpha) \leq \sigma_{w'}(\alpha)$;

2. If $\nabla_{r \mapsto \alpha}^{\sigma}>0$, then $\sigma(\alpha) \geq \sigma_{w'}(\alpha)$.

\end{property}

\begin{proposition}
\label{proposition_counterfactuality_direct}
Let $r$ be a direct or indirect edge w.r.t. $\alpha$.
$\nabla_{r \mapsto \alpha}^{\sigma}$ satisfies counterfactuality if $\sigma$ satisfies both edge-monotonicity and monotonicity.
\end{proposition}

As an illustration, if we set $w(Acting,Movie)=0$, then $\sigma(Movie)$ decreases from $0.827$ to $0.802$.

Qualitative invariability, inspired by \cite{AAE_ECAI,RAE_IJCAI}, states that the influence of an edge on the topic argument remains qualitatively consistent (i.e., always positive or always negative), irrespective of any variations in its weight.
\begin{property}[Qualitative Invariability]
\label{property_qua_inv}
Let $\alpha \in \mathcal{A}$ and $r \in  \mathcal{R}$.
Let $\nabla_{\delta}$ denote $\nabla_{r \mapsto \alpha}^{\sigma}$ when setting $w(r)$ to some $\delta \in [0,1]$.

1. If $\nabla_{r \mapsto \alpha}^{\sigma} < 0$, then $\forall \delta \in [0,1]$, $\nabla_{\delta} \leq 0$;

2. If $\nabla_{r \mapsto \alpha}^{\sigma} > 0$, then $\forall \delta \in [0,1]$, $\nabla_{\delta} \geq 0$.

\end{property}

\begin{proposition}
\label{proposition_qua_inv_direct}
Let $r$ be a direct or indirect edge with respect to $\alpha$.
$\nabla_{r \mapsto \alpha}^{\sigma}$ satisfies qualitative invariability
if $\sigma$ satisfies both edge-monotonicity and monotonicity.
\end{proposition}

As an illustration, even when the edge weights in Figure~\ref{fig_qbaf} are altered, their G-RAEs will always satisfy $\nabla_{r \mapsto \alpha}^{\sigma} \geq 0$ or $\nabla_{r \mapsto \alpha}^{\sigma} \leq 0$.

The final property shows that our G-RAEs can be computed efficiently in linear time with respect to the number of arguments and relations.
\begin{proposition}[Tractability]
\label{proposition_tractability}
Let $|\mathcal{A}|=m$ and $|\mathcal{R}|=n$, then G-RAEs can be generated in linear time $\mathcal{O}(m+n)$ for acyclic EW-QBAFs.
\end{proposition}

In this section, we defined G-RAEs and examined the influence of three types of connectivity, followed by two properties, and finally showed the computational complexity of G-RAEs.
We analysed direct, indirect and independent influence, but did not consider the multifold case. This is because even under the guarantees of monotonicity and edge-monotonicity (which only apply to direct edges), a multifold edge can still have a non-monotonic effect on the topic argument, which may result in the violation of these properties. Such effects may arise in an infinite number of special cases, as discussed in \cite{AAE_ECAI,RAE_IJCAI}. For instance, an edge is involved in multiple paths that convey qualitatively different influences (both positive and negative) on the topic argument. We did not provide guarantees for the multifold case because it requires too many case differentiations about the 
number and strength of different attack and support paths. However, the analysis of the properties for the direct and indirect cases still has much practical value, as various applications are based on tree-like QBAFs (e.g., \cite{cocarascu2019extracting,Rago_23,kotonya2019gradual,chi2021optimized}).


\section{Contestability Algorithms}
\label{sec_algo}
In this section, we propose an approximation algorithm for G-RAEs
that can also be applied to cyclic EW-QBAFs 
and an iterative algorithm for solving the contestability problem.

\begin{algorithm}[h]
\caption{G-RAE Approximation}
\label{algo_GRAE}
\textbf{Input}: An EW-QBAF $\mathcal{Q}$, 
an edge-weighted gradual semantics $\sigma$, 
a topic argument $\alpha$.\\
\textbf{Parameter}: A perturbation value $\varepsilon$.\\
\textbf{Output}: Approximate G-RAEs $g\_rae$.\\
\textbf{Function}: $gRAE(\mathcal{Q},\sigma,\alpha,\varepsilon)$
\begin{algorithmic}[1] 
\STATE $\texttt{g\_rae} \leftarrow \{\}$ \hfill \% attribution scores
\STATE $\texttt{s}_0 \leftarrow \sigma(\alpha)$ \hfill \% original strength
\FOR {\texttt{r} in $\mathcal{R}$}
    \STATE $\texttt{w[r]} \leftarrow \texttt{w[r]}+\varepsilon$ \hfill \% perturb $w(r)$\\
    \STATE $s' \leftarrow \sigma(\alpha)$ \hfill \% perturbed strength\\
    \STATE $\texttt{g\_rae[r]} \leftarrow (\texttt{s}'-\texttt{s}_0)/\varepsilon$ \hfill \% compute G-RAE\\
    \STATE $\texttt{w[r]} \leftarrow \texttt{w[r]}-\varepsilon$ \hfill \% restore $w(r)$\\
\ENDFOR
\STATE \textbf{return} $\texttt{g\_rae}$ \hfill 
\end{algorithmic}
\end{algorithm}

To simplify the implementation and avoid handling the structure variations of different EW-QBAFs, we adopt a perturbation-based method to approximate G-RAEs. Perturbation-based approaches are also widely used in the literature for gradient estimation tasks when the gradient cannot be 
computed analytically (e.g., \cite{ozbulak2020perturbation,minervini2023adaptive}).
Algorithm~\ref{algo_GRAE} computes the G-RAE for all edges with respect to the topic argument in an EW-QBAF. The algorithm begins by computing the original strength of the topic argument $\alpha$ under the gradual semantics $\sigma$ (line 2). 
Then, for an edge $r$, the algorithm perturbs $w(r)$ by a value $\varepsilon$ and recomputes $\sigma(\alpha)$ based on the perturbed edge weight (line 4-5). The approximate G-RAE of $r$ is then obtained by dividing the strength change of $\alpha$ by $\varepsilon$ (line 6). After computing the G-RAE of edge $r$, the original weight $w(r)$ is restored for proceeding to the next edge (line 7). The above procedure is iteratively applied to compute the G-RAE for all edges. 

\begin{proposition}[G-RAE Approximation Complexity]
\label{proposition_approx_GRAEs}
Let $|\mathcal{A}|=m$ and $|\mathcal{R}|=n$, then all approximate G-RAEs can be generated in time $\mathcal{O}(n \cdot (m+n))$ for acyclic EW-QBAFs.
\end{proposition}

\begin{algorithm}[t]
    \caption{Contestability Algorithm}
    \label{algo_contest}
    \textbf{Input}: An EW-QBAF $\mathcal{Q}$, an edge-weighted gradual semantics $\sigma$, a topic argument $\alpha$ and a desired strength $s$ for $\alpha$.\\
    \textbf{Parameter}: A maximum iteration limit \texttt{M}, an updating step \texttt{h}, an error threshold $\delta$.\\
    \textbf{Output}: An edge weight function $\texttt{w}'$.
    
    \begin{algorithmic}[1] 
    \STATE $\texttt{w}'$ $\leftarrow$ \texttt{w}
        \STATE $\texttt{s}' \leftarrow \sigma(\alpha)$ \hfill
        \WHILE{$|\texttt{s}'- \texttt{s}|>\delta$ and $\texttt{M} > \texttt{0}$}
            \STATE $\texttt{g\_rae} \leftarrow gRAE(\mathcal{Q},\sigma,\alpha,\varepsilon)$
            \FOR {\texttt{r} in $\mathcal{R}$}
                \STATE $\texttt{update} \leftarrow \texttt{w}'\texttt{[r]}+\texttt{g\_rae[r]} \cdot \texttt{h}$
                \STATE $\texttt{w}'\texttt{[r]} \leftarrow \texttt{max}(\texttt{0},  \texttt{min}(\texttt{1}, \texttt{update}))$
            \ENDFOR
            \STATE $\texttt{s}' \leftarrow \sigma(\alpha)$
            \STATE $\texttt{M} \leftarrow \texttt{M}-\texttt{1}$
        \ENDWHILE
        \STATE \textbf{return} $\texttt{w}'$
    \end{algorithmic}
\end{algorithm}

Since G-RAEs quantify the influence of edge weights on the strength changes of the topic argument, we utilize them to guide an iterative algorithm that adjusts edge weights to solve the contestability problem. Algorithm~\ref{algo_contest} begins by computing the strength of the topic argument $\alpha$ under the gradual semantics $\sigma$ (line 2). For brevity, we assume the current strength of the topic argument is less than the desired one. 
If the difference between this strength and the desired strength exceeds a predefined error threshold $\delta$ and the iteration count has not yet reached a maximum iteration limit $M$ (introduced to prevent the algorithm from running indefinitely or getting stuck in a local minimum) (line 3), the edge weights are updated by adding their respective G-RAEs multiplied by an updating step $h$ (line 6). Since the edge weights are constrained within the range $[0,1]$, we apply the $\max$ and $\min$ functions to ensure the updated weights remain within bounds (line 7). Once all edge weights are updated, $\sigma(\alpha)$ is recomputed and the iteration count $M$ is decremented (line 9-10). The algorithm terminates when either $\sigma(\alpha)$ is sufficiently close to the desired strength or the iteration limit is reached.
The time complexity of this algorithm is shown as follows.

\begin{proposition}[Contestability Algorithm Complexity]
\label{proposition_iterative_algo}
Let $|\mathcal{A}|=m$ and $|\mathcal{R}|=n$, and suppose the maximum number of iterations is $M$. Then the time complexity of Algorithm~\ref{algo_contest} is $\mathcal{O}(M \cdot n \cdot (m+n))$ for acyclic EW-QBAFs.
\end{proposition}

\section{Experimental Evaluation}
\label{sec_eval}

\begin{figure*}[t]
    \centering
    \includegraphics[width=0.9\linewidth]{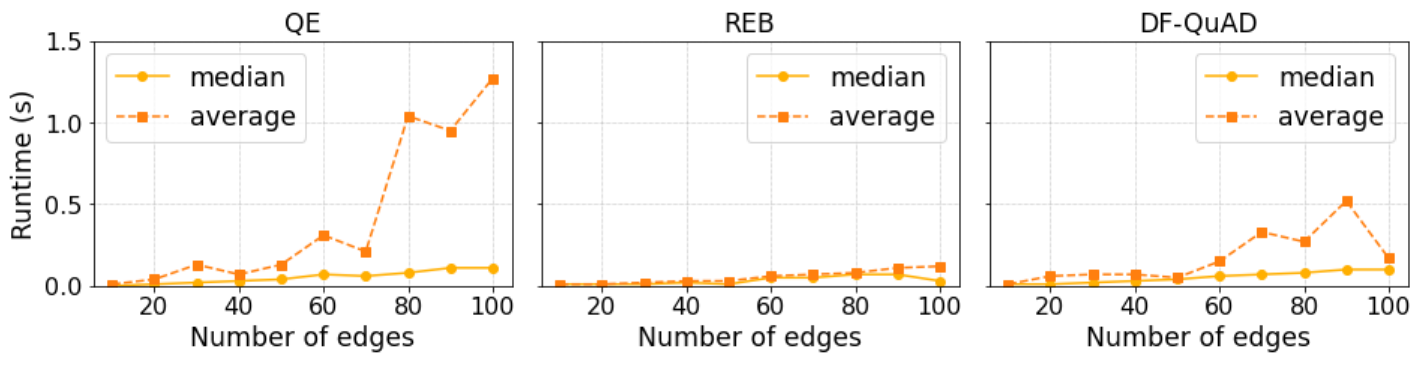}
    \caption{Average and median runtime over 100 randomly generated PRS-like QBAFs under different semantics and edge sizes.}
    \label{fig_result1}
\end{figure*}

\begin{figure*}[t]
    \centering
    \includegraphics[width=0.9\linewidth]{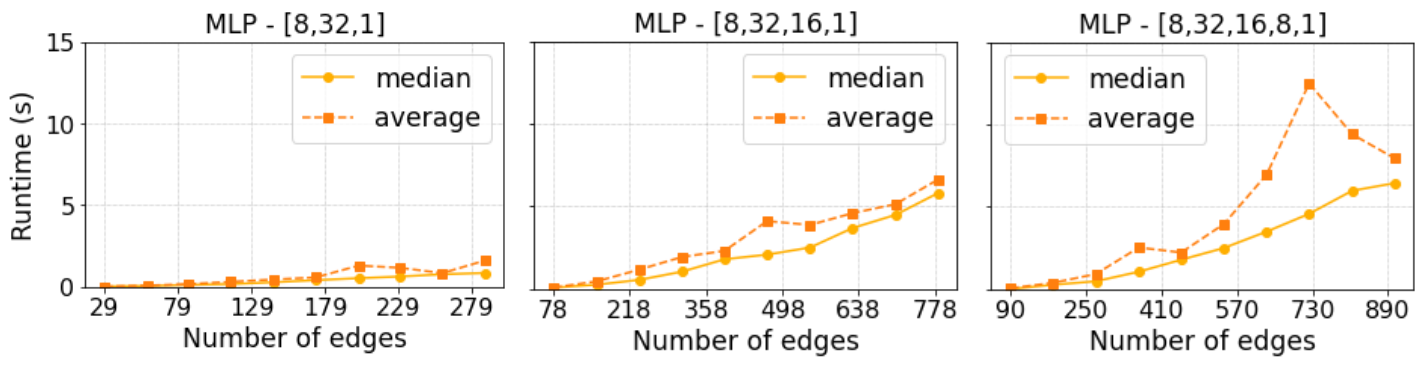}
    \caption{Average and median runtime over 100 randomly generated MLP-like QBAFs, evaluated under MLP-based semantics across various structures and edge sizes.}
    \label{fig_result2}
\end{figure*}

We evaluated the effectiveness (ability to achieve the desired strength) and scalability (performance on dense EW-QBAFs) of our contestability algorithm through two potential application scenarios: PRSs (Experiment 1) and MLP debugging (Experiment 2).\footnote{For detailed results in this section, see Table~\ref{tab111} and \ref{tab222} in the SM.}

\paragraph{Algorithm Parameter Settings}
In Algorithm~\ref{algo_GRAE}, the perturbation value $\varepsilon$ was set to $10^{-5}$ for approximating the G-RAE of edges.
In Algorithm~\ref{algo_contest}, the error threshold $\delta$ was set to $0.01$ and the maximum number of iterations $M$ was set to $1000$. To accelerate convergence, we employed a dynamic updating step schedule: a larger step was applied when the current strength was farther from the desired strength, and a smaller step was used when the current strength is closer to the desired strength.

\subsection{Experiment 1}
\label{expr1}
In this experiment, we aimed to evaluate the performance of our algorithm on acyclic EW-QBAFs that mirror the structure of PRSs to show the potential of addressing misalignments between the EW-QBAF outcomes and human expectations.
We assessed the \textbf{effectiveness} of the algorithm in this setting.

\paragraph{Setups}
In order to randomly generate acyclic EW-QBAFs, we first created $n$ arguments $(\alpha_1,\cdots,\alpha_n)$, each assigned a random base score drawn uniformly from the interval [0,1]. For every pair of arguments $(\alpha_i,\alpha_j)$ such that $i<j$, an edge was generated with probability $p$.
Since QBAFs in many applications typically contain fewer than 100 arguments and have a roughly equal numbers of arguments and edges (e.g., movie recommender systems~\cite{cocarascu2019extracting}, fake news detection~\cite{kotonya2019gradual}, fraud detection~\cite{chi2021optimized}), we generated EW-QBAFs of varying size with $n \in \{10,20,\dots,100\}$ and $p=\frac{2}{n}$.
Each edge was randomly designated as either attack or support with equal probability, and its weight was uniformly drawn from the interval [0,1]. The topic argument was designated as $\alpha_n$, and the desired strength for it was set to the mean of the maximum and minimum attainable strength. This choice was motivated by Theorem~\ref{theorem2}, which guarantees that the mean value is always attainable. 
Note that, depending on $p$, the topic argument may have no or only a few predecessors in this setting.
To mitigate the effects of randomness, we generated 100 QBAFs for each size. To compute argument strengths, we applied three commonly considered gradual semantics - QE~\cite{Potyka18},  REB~\cite{amgoud2018evaluation}, and DF-QuAD~\cite{rago2016discontinuity} - each adapted to incorporate edge weights to suit our setting.

\paragraph{Results and Analysis}
We evaluated \emph{validity} (i.e., whether the desired strength was attained), \emph{runtime}, and number of \emph{attempts} of our algorithm across different semantics and varying edge sizes.
We begin by analysing effectiveness. Across all 100 EW-QBAFs of each structure, the validity of our algorithm consistently reached $100\%$, regardless of semantics or the number of edges, indicating that it was always able to successfully identify the desired strength, in accordance with Theorem~\ref{theorem2}.

Next, we discuss the number of attempts and the algorithm running time.
Since our algorithm is based on gradient descent, there exists a potential risk of converging to a local minimum. To address this, whenever the solution was not found within the maximum number of iterations $M$, we adjusted the initial search point and re-executed the algorithm. 
Under the REB semantics, our algorithm consistently succeeded on the first attempt across all EW-QBAF sizes, followed by the QE semantics with an average of $1.01$ attempts.
Although the DF-QuAD semantics required the highest average number of attempts ($1.014$), the maximum never exceeded $4$ - matching that of the QE semantics - demonstrating that the algorithm can find solutions within reasonable numbers of attempts across all cases.

Regarding the runtime performance shown in Figure~\ref{fig_result1}, under the same EW-QBAF structure, the REB semantics generally exhibited the shortest average runtime, followed by DF-QuAD, while QE incurred the longest. Despite these differences, the average runtime across all semantics and all EW-QBAF sizes remained under 1.5 seconds.
Due to the random generation of EW-QBAFs, certain instances required significantly more time to process. Thus, we also reported the median runtime.
Across all three semantics, the median runtimes were comparable and exhibited an overall increasing trend with minor fluctuations\footnote{Minor measurement inaccuracies when execution time falls below 0.1 seconds may contribute to the observed fluctuations.} as the number of edges increased. In all cases, the median runtime remained below $0.12$ seconds.
Overall, the results on validity, runtime, and the number of attempts collectively demonstrate that our algorithm is both effective and runs in polynomial time across a wide range of acyclic QBAF configurations, highlighting its potential applicability to address misalignments between users' expectations and the EW-QBAFs' outcomes in the PRS setting.

\subsection{Experiment 2}
\label{expr2}
In this experiment, we aimed to evaluate our algorithm on denser QBAFs that mirror the structure of MLPs. As shown in \cite{potyka2021interpreting}, MLPs can be translated into EW-QBAFs, and our algorithm has the potential to assist in debugging MLPs by achieving a desired output for a given instance.
In addition to effectiveness and polynomial complexity, we further evaluate the \textbf{scalability} of our algorithm.

\paragraph{Setups}
We generated denser QBAFs compared to those in Experiment 1. 
Specifically, we designed three MLP-like QBAFs structures: $[8,32,1]$, $[8,32,16,1]$ and $[8,32,16,8,1]$, varying in the numbers of neurons and hidden layers. For example, $[8,32,1]$ consists of 8 input arguments, three hidden layers with 32, 16, and 8 arguments respectively, and 1 output argument, yielding a total of $8+32+16+8+1=65$ arguments and $8\times32+32\times16+16\times8+8\times1=904$ edges in the fully-connected case.
We chose these structures because they are suitable for many binary classification tasks such as Pima Indians diabetes classification problem~\cite{potyka2022towards}.
For each structure, the base scores of arguments, as well as the polarity and weights of edges, were assigned in the same manner as in Experiment 1. The topic argument was set as the output-layer argument.
We varied the connection density by setting the probability $p$ of an edge between any two arguments in adjacent layers, with $p \in \{0.1, 0.2, \dots, 1.0\}$. A value of $p=1.0$ amounted to a fully-connected MLP-like EW-QBAF.
We generated 100 EW-QBAFs for each structure to avoid randomness and applied MLP-based gradual semantics~\cite{potyka2021interpreting} to compute argument strengths.

\paragraph{Results and Analysis}
We begin with checking the effectiveness of our algorithm. The validity consistently reached $100\%$ across all EW-QBAF structures, indicating its ability to attain the desired strengths, even in significantly denser configurations compared to those in Experiment 1.
We next evaluate the number of attempts and running time.
For all structures, our algorithm succeeded on the first attempt, except for the case of $[8, 32, 1]$, where a single fully-connected ($p=0.7$) MLP-like EW-QBAF required one additional attempt.
Overall, the results outperformed those of Experiment 1. A possible explanation lies in the structural constraint imposed in this setting: each argument only connects to arguments in the next layer, resulting a simpler and more hierarchical structure. In contrast, Experiment 1 allowed a more flexible connectivity.
As shown in Figure~\ref{fig_result2}, the average runtime increased with the number of edges across all structures, with the maximum average reaching $12.50$ seconds. Meanwhile, the median runtime exhibited a strictly increasing trend, peaking at $6.44$ seconds. The polynomial runtime results further showed our algorithm remained scalable, even as the density of EW-QBAFs increases.
In summary, the results on MLP-like EW-QBAFs demonstrate the effectiveness and scalability of our algorithm, suggesting its potential applicability for debugging MLPs.

\section{Conclusion}
In this work, we investigated the potential of EW-QBAFs for supporting contestability. We began by introducing new properties for EW-QBAFs that explain the effect of edge weight changes.
Then, we formally defined the contestability problem for acyclic EW-QBAFs and discussed the attainability of its solution. To address this problem, we introduced a novel notion of G-RAEs, which provide interpretable guidance towards contestability. We also studied and adapted several desirable properties from existing argumentative attribution explanations for standard QBAFs to the context of G-RAEs. Based on G-RAEs, we developed an iterative algorithm that progressively adjusts the edge weights to attain a desired strength. We empirically evaluated this algorithm in terms of effectiveness, runtime complexity and scalability on synthetic EW-QBAFs that simulates PRS and MLP-debugging scenarios, demonstrating its potential applicability.

Future research could focus on the following directions. Firstly, exploring the applicability of our theory in real-world EW-QBAFs scenarios, particularly in the contexts of PRS contestation and MLP debugging. 
Secondly, conducting user studies to evaluate the effectiveness and usability of our theory through user feedback, with a focus on how well users can understand, trust, and act upon the explanations provided, thereby assessing the role of explanations in supporting contestability.
Finally, defining and investigating the \emph{multi-contestability} problem, where multiple base score functions yield both desired and undesired strengths for a given topic argument, and the goal is to revise edge weights in a way that collectively ensures all (or as many as possible) of the desired strengths are achieved.

\section*{Acknowledgments}
This research was partially funded by the  European Research Council (ERC) under the
European Union’s Horizon 2020 research and innovation programme (grant
agreement No. 101020934, ADIX) and by J.P. Morgan and by the Royal
Academy of Engineering under the Research Chairs and Senior Research
Fellowships scheme.  Any views or opinions expressed herein are solely those of the authors.

\bibliographystyle{kr}
\bibliography{kr25}

\newpage
\setcounter{page}{1}
\onecolumn
\appendix

\section*{Supplementary Material for\\``Contestability in Edge-Weighted
Quantitative Bipolar Argumentation Frameworks''}
\medskip

\section{Proofs}
\setcounter{proposition}{0}
\setcounter{theorem}{0}
\setcounter{lemma}{0}
\setcounter{observation}{0}

\begin{lemma}
If a modular semantics is based on an aggregation function $\agg$ that satisfies neutrality, then
the semantics satisfies Edge-neutrality.
\end{lemma}
\begin{proof}
To prove the claim, we show that the strength functions $\sigma^{(i)}$ and $\sigma_{\mathcal{Q}'}^{(i)}$ 
computed in the iterative computation of $\sigma$ and $\sigma_{\mathcal{Q}'}$ are equal,
which implies that their limit must be equal.
We can assume w.l.o.g. that $(\beta,\alpha) \in \mathcal{R^{-}}'(\alpha)$ (the case $(\beta,\alpha) \in \mathcal{R^{+}}'(\alpha)$ is analogous).
We prove the claim by induction.
For $i=0$, we have 
$\sigma^{(0)}(\alpha) = \tau(\alpha) = \sigma_{\mathcal{Q}'}^{(0)}(\alpha)$.
For the induction step, we have
$\sigma^{(i+1)}(\alpha) 
= \infl(\tau(\alpha), \agg(A^{\sigma^{(i)}}_\alpha,S^{\sigma^{(i)}}_\alpha))
= \infl(\tau(\alpha), \agg(A^{\sigma^{(i)}}_\alpha \cup \{\sigma^{(i)}_{\mathcal{Q}'}(\beta) \cdot w'((\beta, \alpha))\},S^{\sigma^{(i)}}_\alpha))$,
where the last equality follows from the assumptions that $w'((\beta,\alpha))=0$
and neutrality of $\agg$. Using the induction assumption, we have 
$\agg(A^{\sigma^{(i)}}_\alpha \cup \{\sigma^{(i)}_{\mathcal{Q}'}(\beta) \cdot w'((\beta, \alpha))\},S^{\sigma^{(i)}}_\alpha)) = 
\agg(A^{\sigma_{\mathcal{Q}'}^{(i)}}_\alpha, S^{\sigma_{\mathcal{Q}'}^{(i)}}_\alpha )$, and therefore 
$\sigma^{(i+1)}(\alpha) = \sigma_{\mathcal{Q}'}^{(i+1)}(\alpha) $.
\end{proof}
\begin{lemma}
If a modular semantics is based on an aggregation function $\agg$
and an influence function $\infl$ which satisfy their associated balance properties, then
the semantics satisfies Edge-stability.
\end{lemma}
\begin{proof}
To prove the claim, we show that it holds for all strength functions $\sigma^{(i)}$
computed in the iterative computation of $\sigma$, 
which implies that it also holds in the limit.
We prove the claim by induction.
For $i=0$, we have 
$\sigma^{(0)}(\alpha) = \tau(\alpha)$.
For the induction step, we have
$\sigma^{(i+1)}(\alpha) 
= \infl(\tau(\alpha), \agg(A^{\sigma^{(i)}}_\alpha,S^{\sigma^{(i)}}_\alpha))
= \infl(\tau(\alpha), 0)
= \tau(\alpha).
$
For the second inequality, we used the fact that  
$0 = \sigma^{(i)}(\beta) \cdot w((\beta, \alpha))$ for all 
$(\beta, \alpha) \in \mathcal{R}(\alpha)$. Hence, 
$A^{\sigma^{(i)}}_\alpha \cong S^{\sigma^{(i)}}_\alpha$,
and the equality follows from balance of $\agg$.
The last equality follows from balance of $\infl$.
\end{proof}
\begin{lemma}
Every modular semantics satisfies Edge-directionality.    
\end{lemma}
\begin{proof}
Similar to the previous proofs, the claim follows by observing that the iterative 
update of $\gamma$ is equal with respect to both $\mathcal{Q}$ and $\mathcal{Q}'$.
To see this, just note that the aggregation function takes only the parents of an argument into account. Since $P_{\beta \mapsto \gamma} = \emptyset$,
$\gamma$ and its predecessors will be evaluated in the exact same way in  $\mathcal{Q}$ and $\mathcal{Q}'$. Hence, their strength values
must be equal for both graphs.
\end{proof}
\begin{lemma}
If a modular semantics is based on an aggregation function $\agg$
and an influence function $\infl$ which satisfy their associated monotonicity properties, then
the semantics satisfies Monotonicity and Edge-Monotonicity for the class of acyclic EW-QBAFs.    
\end{lemma}\begin{proof}
To prove the claims, we first note that,
for acyclic graphs, the iterative update procedure of strength values is equivalent to a single forward-pass
with respect to an arbitrary topological ordering \cite[Proposition 3.1]{Potyka19}.  
Fix any such ordering. Since we have the  edge $(\beta,\alpha)$, $\beta$ will occur
before $\alpha$. Since  $\tau$ and $\tau'$ differ only for $\beta$,
the evaluation of all arguments up to $\beta$ will be equal in $\mathcal{Q}$ and $\mathcal{Q}'$.
Furthermore, since $\tau$ and $\tau'$ differ only for $\beta$ and there is 
only a single path from $\beta$ to $\alpha$ (via the edge $(\beta,\alpha)$),
every argument $\gamma$ that is a predecessor of $\alpha$ (there is a directed path
from $\gamma$ to $\alpha$)  does not have $\beta$ as a predecessor (otherwise, there would
be a second path from $\beta$ to $\alpha$ via $\gamma$).
Hence, after evaluating all arguments up to $\alpha$ according to the topological ordering,
the strength of $\alpha$'s parents (up to $\beta$) in $\mathcal{Q}$ and $\mathcal{Q}'$
will be equal. We prove the first item of the monotonicity definition.

We have  
$\sigma(\alpha) 
= \infl(\tau(\alpha), \agg(A^{\sigma}_\alpha, S^{\sigma}_\alpha))$
and 
$\sigma_{\tau'}(\alpha) 
= \infl(\tau(\alpha), \agg(A^{\sigma_{\tau'}}_\alpha,S^{\sigma_{\tau'}}_\alpha))$.
Our previous considerations imply that $S^{\sigma}_\alpha = S^{\sigma_{\tau'}}_\alpha$,
and that $A^{\sigma}_\alpha$ and $A^{\sigma_{\tau'}}_\alpha$ can differ only in the 
evaluation of $\beta$. However, since the evaluation of all parents of $\alpha$ except
$\beta$ is
equal in $\mathcal{Q}$ and $\mathcal{Q}'$, a difference can only occur because of the
strength of $\beta$. By assumption, $\beta$'s base score in 
$\mathcal{Q}$ is bounded from above by the base score in $\mathcal{Q}'$.
Hence, monotonicity of the aggregation and influence function implies that the 
strength of $\beta$ in $\mathcal{Q}$ is bounded by the base score in $\mathcal{Q}'$.
This implies $A^{\sigma_{\tau'}}_\alpha \succeq  A^{\sigma}_\alpha$.
Hence,
$\agg(A^{\sigma}_\alpha, S^{\sigma}_\alpha) \geq 
\agg(A^{\sigma_{\tau'}}_\alpha,S^{\sigma_{\tau'}}_\alpha)$
by monotonicity of the aggregation function and thus
$\infl(\tau(\alpha), \agg(A^{\sigma}_\alpha, S^{\sigma}_\alpha))
\geq \infl(\tau(\alpha), \agg(A^{\sigma_{\tau'}}_\alpha,S^{\sigma_{\tau'}}_\alpha))
$
by monotonicity of the influence function.
Hence, $\sigma(\alpha) \geq \sigma_{\tau'}(\alpha)$.

The proof of the second item for monotonicity is analogous.
For edge-monotonicity, we can make similar considerations. 
Similar to monotonicity, there is only one difference in the
QBAFs. In this case, it is the weight of the edge instead of the
base score of the source argument of the edge. We can therefore
make the same considerations using the edge weight instead of the
base score.
\end{proof}
\begin{proposition}
\label{prop_mono_sta_dir}
Edge-Weighted QE, REB, DF-QuAD and MLP-based semantics satisfy 
edge-neutrality,
edge-stability,
edge-directionality,
monotonicity,
edge-monotonicity.
\end{proposition}
\begin{proof}
Follows from the fact that the aggregation and influence functions of the semantics satisfy their associated balance,
monotonicity and neutrality properties \cite[Lemma 10, 14]{PotykaB24}, and our previous lemmas.    
\end{proof}
\begin{lemma}
For acyclic EW-BAFS, the strength function under Edge-Weighted QE, REB, DF-QuAD and MLP-based semantics
is differentiable with respect to edge-weights.
\end{lemma}
\begin{proof}
The proof is based on the fact that differentiable functions are closed under function composition
and products.
We first note that the aggregation and influence function of the Edge-Weighted QE, REB, DF-QuAD and MLP-based semantics
are differentiable with respect to edge-weights because they are composed of elementary functions that are differentiable
with respect to edge-weights.

Now consider an acyclic EW-QBAFs interpreted by a modular semantics whose aggregation and influence functions
are differentiable with respect to edge-weights. The strength of every argument can be written in closed form
as a finite composition of these aggregation and influence functions. A constructive induction proof can be 
sketched as follows: since the graph is acyclic, we can compute a topological ordering of the 
arguments. We start from the arguments without parents and successively move to their successors
following the topological order. We start from arguments without parents. Then the composed function
is of the form $\infl(\tau(\alpha), \agg(\emptyset, \emptyset))$. The function does not depend on
any edge-weights and is therefore trivially differentiable with respect to edge weights (the derivative is $0$).
Hence, it's composed of edge-differentiable functions. For the induction step, we can assume that the strength function of all
predecessors of the argument can be written as a composition of edge-differentiable functions by the induction
assumption. Our function
for the new argument is now
 $\infl(\tau(\alpha), \agg(A^\sigma_\alpha, S^\sigma_\alpha))$, where the elements in 
$A^\sigma_\alpha \cup S^\sigma_\alpha$ are of the form
$\sigma(\beta) \cdot w((\beta, \alpha))$. 
$\sigma(\beta)$ is edge-differentiable by assumption and $w((\beta, \alpha))$
can be seen as the identity function mapping an edge weight to itself. It is
therefore an edge-differentiable function. Since differentiable functions are closed under
product, $\sigma(\beta) \cdot w((\beta, \alpha))$ is an edge-differentiable function.
Hence, since differentiable functions are closed under function composition,
$\infl(\tau(\alpha), \agg(A^\sigma_\alpha, S^\sigma_\alpha))$ is an edge-differentiable function.
\end{proof}

\begin{observation}[Solution Existence]
For any $\alpha \in \mathcal{A}$ and any desired strength $s \in [0,1]\setminus\{\sigma(\alpha)\}$ for $\alpha$, the solution to the contestability problem exists iff $s \in S(\alpha)$. 
\end{observation}
\begin{proof}
By Definition~\ref{def_contest}, the existence of a solution to the contestability problem implies that $\exists w'$ such that $\sigma_{w'}(\alpha) = s$ , meaning $s$ is attainable for $\alpha$, i.e., $s \in S(\alpha)$, by Definition~\ref{def_attainability}.
Conversely, if $s \in S(\alpha)$, then by Definition~\ref{def_attainability}, $s$ is attainable for $\alpha$, which implies that $\exists w'$ such that $\sigma_{w'}(\alpha) = s$. Therefore, by Definition~\ref{def_contest}, $w'$ constitutes a solution to the contestability problem.
\end{proof}

\begin{proposition}[Base Score Attainability]
For any $\alpha \in \mathcal{A}$, 
if $\sigma$ satisfies edge-stability,
then  $\tau(\alpha) \in S(\alpha)$.
\end{proposition}
\begin{proof}
Consider an edge weight function $w'$ such that $w'(r) = 0$ for all $r \in \mathcal{R}(\alpha)$. Then, by Definition \ref{def_edge_stability}, we have $\sigma_{w'}(\alpha)=\tau(\alpha)$. Thus, by Definition~\ref{def_attainability}, $\tau(\alpha)$ is attainable, i.e., $\tau(\alpha) \in S(\alpha)$.
\end{proof}

\begin{theorem}[Maximum and Minimum of Attainability]
For any \( \alpha \in \mathcal{A} \) and any $\sigma$ satisfying edge-monotonicity, monotonicity, and edge-neutrality, \( \sigma_{w^{\alpha}_{\max}}(\alpha) \) and \( \sigma_{w^{\alpha}_{\min}}(\alpha) \) are the maximum and minimum of \( S(\alpha) \), respectively.
\end{theorem}
\begin{proof}
We first show that $\sigma_{w^{\alpha}_{\max}}(\alpha)$ is the maximum of $S(\alpha)$, with three steps. Figure~\ref{fig_max_min} (left) serves as an illustration to support the intuition.

First, by edge-monotonicity and edge-neutrality, in order to maximise $\sigma(\alpha)$, we need to set the edge weights of all supports directly incoming to $\alpha$ to 1 and those of all direct attacks to 0.
For instance, $w(r_{10})=w(r_{12})=1$ and $w(r_{11})=0$.

Second, by monotonicity, we need to maximise the strength of the direct supporters and attackers of $\alpha$.
\begin{itemize}
    \item For each supporter of $\alpha$, we maximise its strength by setting the weights of its incoming supports to 1 and its incoming attacks to 0. For instance, $w(r_4)=1$ and $w(r_8)=0$.
    \item For each attacker of $\alpha$, we also maximise its strength similarly. This is because attackers may support the supporters of $\alpha$, thereby indirectly increasing $\sigma(\alpha)$. Meanwhile, their direct attacks on $\alpha$ or on the supporters of $\alpha$ have already been eliminated by setting their edge weights to 0. For instance, $w(r_6)=w(r_7)=w(r_9)=1$ and $w(r_5)=0$.
\end{itemize}

Finally, we recursively maximise the strength of all remaining arguments in $\mathcal{Q}$ by applying the same principle: maximising the weights of incoming supports to 1 and minimising the weights of incoming attacks to 0.

By this construction, every support edge in $\mathcal{Q}$ is assigned weight 1 and every attack edge is assigned weight 0, leading to the maximal attainable strength $\sigma_{w^{\alpha}_{\max}}(\alpha)$ for $\alpha$.

The proof for $\sigma_{w^{\alpha}_{\min}}(\alpha)$ being the minimum of $S(\alpha)$ proceeds analogously, except for the first step, where we minimise the weights of all supports of $\alpha$ to 0 and maximise the weights of all attacks of $\alpha$ to 1. The second and third step remain the same, as we still aim to maximise the strength of the attackers of $\alpha$.

Therefore, $\sigma_{w^{\alpha}_{\max}}(\alpha)$ and $\sigma_{w^{\alpha}_{\min}}(\alpha)$ are the maximum and minimum of $S(\alpha)$, respectively, which completes the proof.
\end{proof}

\begin{theorem}[Completeness of Attainability]
For any \( \alpha \in \mathcal{A} \),
let $m$ and $M$ denote the minimum and maximum of $S(\alpha)$, respectively. If $\sigma$ satisfies continuity, then $S(\alpha) =[m,M]$.
\end{theorem}
\begin{proof}
Since both $\min$ and $\max$ are attainable for $\alpha$, and given that $\sigma$ is differentiable and thus continuous with respect to edge weights, 
by the Intermediate Value Theorem, any $s \in (\min,\max)$ must also be attainable for $\alpha$. Therefore, we conclude that the set of attainable strength for $\alpha$ is given by $S(\alpha) = [\min,\max]$.
\end{proof}

\begin{proposition}[Direct Influence]
Let $r \in \mathcal{R}$ be a direct edge w.r.t. $\alpha \in \mathcal{A}$. Then the following statements hold if $\sigma$ satisfies edge-monotonicity:

1. If $r \in \mathcal{R^{+}}$, then $\nabla_{r \mapsto \alpha}^{\sigma} \geq 0$;

2. If $r \in \mathcal{R^{-}}$, then $\nabla_{r \mapsto \alpha}^{\sigma} \leq 0$.

\end{proposition}
\begin{proof}
We first consider the case when $r \in \mathcal{R^{+}}$.
Let $w'$ be an edge weight function such that $w'(r) = w(r)+\varepsilon$ and $w(t) = w'(t)$ for all $t \in \mathcal{R} \setminus \{r\}$.\\
\begin{itemize}
\item Let $\varepsilon \rightarrow 0^+$. Then, we have $w'(r) > w(r)$.
Since $\sigma$ satisfies edge-monotonicity and $r \in \mathcal{R}^{+}(\alpha)$, we have:
\[
\sigma_{w'}(\alpha) \geq \sigma(\alpha) \Rightarrow \frac{\sigma_{w'}(\alpha) - \sigma(\alpha)}{\varepsilon} \geq 0.
\]

\item Let $\varepsilon \to 0^-$. Then, we have $w'(r) < w(r)$. 
Similarly, we have:
\[
\sigma_{w'}(\alpha) \leq \sigma(\alpha) \Rightarrow \frac{\sigma_{w'}(\alpha) - \sigma(\alpha)}{\varepsilon} \geq 0.
\]
\end{itemize}

Since $\frac{\sigma_{w'}(\alpha) - \sigma(\alpha)}{\varepsilon} \geq 0$ in both directions, the limit exists and satisfies:
\[
\nabla_{r \mapsto \alpha}^{\sigma} = \lim_{\varepsilon \to 0} \frac{\sigma_{w'}(\alpha) - \sigma(\alpha)}{\varepsilon} \geq 0.
\]

The case where $r \in \mathcal{R}^{-}$ follows analogously and we have:
\[
\nabla_{r \mapsto \alpha}^{\sigma} = \lim_{\varepsilon \to 0} \frac{\sigma_{w'}(\alpha) - \sigma(\alpha)}{\varepsilon} \leq 0.
\]
\end{proof}

\begin{proposition}[Indirect Influence]
Let $r$ be an indirect edge w.r.t. $\alpha$. Suppose the path sequence from $r$ to $\alpha$ is $\langle r, r_{1}, \cdots, r_{n} \rangle (n \geq 1)$. Let $\lambda= \left| \{r_{1}, \cdots, r_{n} \} \cap \mathcal{R}^{-} \right|$. Then the following statements hold if $\sigma$ satisfies both monotonicity and edge-monotonicity.

    1. If $r \in \mathcal{R^{+}}$ and $\lambda$ is odd, then $\nabla_{r \mapsto \alpha}^{\sigma} \leq 0$;
    
    2. If $r \in \mathcal{R^{+}}$ and $\lambda$ is even, then $\nabla_{r \mapsto \alpha}^{\sigma} \geq 0$;

    3. If $r \in \mathcal{R^{-}}$ and $\lambda$ is odd, then $\nabla_{r \mapsto \alpha}^{\sigma} \geq 0$;
    
    4. If $r \in \mathcal{R^{-}}$ and $\lambda$ is even, then $\nabla_{r \mapsto \alpha}^{\sigma} \leq 0$.
\end{proposition}
\begin{proof}
We first consider the case when $r \in \mathcal{R^{+}}$ and $\lambda$ is odd. To facilitate proof, let us assume $\beta,\gamma\in\mathcal{A}$ and $r=(\beta,\gamma)$.
Let $w'$ be an edge weight function such that $w'(r) = w(r)+\varepsilon$ and $w(t) = w'(t)$ for all $t \in \mathcal{R} \setminus \{r\}$.\\
\begin{itemize}
\item Let $\varepsilon \rightarrow 0^+$. Then, we have $w'(r) > w(r)$.
Since $\sigma$ satisfies edge-monotonicity and $r \in \mathcal{R}^{+}(\alpha)$, we know $\sigma(\gamma)$ will not be decreased.
By monotonicity, the strength change of $\gamma$ will not increase $\sigma(\alpha)$ if the path from $\gamma$ to $\alpha$ has odd number of attacks. Therefore, we have \[
\sigma_{w'}(\alpha) \geq \sigma(\alpha) \Rightarrow \frac{\sigma_{w'}(\alpha) - \sigma(\alpha)}{\varepsilon} \geq 0.
\]
\item Let $\varepsilon \rightarrow 0^-$. Then, we have $w'(r) > w(r)$. Similarly, we have 
\[
\sigma_{w'}(\alpha) \leq \sigma(\alpha) \Rightarrow \frac{\sigma_{w'}(\alpha) - \sigma(\alpha)}{\varepsilon} \geq 0.
\]
\end{itemize}

Since $\frac{\sigma_{w'}(\alpha) - \sigma(\alpha)}{\varepsilon} \geq 0$ in both directions, the limit exists and satisfies:
\[
\nabla_{r \mapsto \alpha}^{\sigma} = \lim_{\varepsilon \to 0} \frac{\sigma_{w'}(\alpha) - \sigma(\alpha)}{\varepsilon} \geq 0.
\]

The other three cases follow analogously and, therefore, we have:

    1. If $r \in \mathcal{R^{+}}$ and $\lambda$ is odd, then $\nabla_{r \mapsto \alpha}^{\sigma} \leq 0$;
    
    2. If $r \in \mathcal{R^{+}}$ and $\lambda$ is even, then $\nabla_{r \mapsto \alpha}^{\sigma} \geq 0$;

    3. If $r \in \mathcal{R^{-}}$ and $\lambda$ is odd, then $\nabla_{r \mapsto \alpha}^{\sigma} \geq 0$;
    
    4. If $r \in \mathcal{R^{-}}$ and $\lambda$ is even, then $\nabla_{r \mapsto \alpha}^{\sigma} \leq 0$.
    
\end{proof}

\begin{proposition}[Irrelevance]
Let $\alpha \in \mathcal{A}$. If $r \in \mathcal{R}$ is an independent edge w.r.t.
$\alpha$ and $\sigma$ satisfies edge-directionality, then $\nabla_{r \mapsto \alpha}^{\sigma} = 0$.
\end{proposition}
\begin{proof}
Let $w'$ be an edge weight function such that $w(t) = w'(t)$ for all $t \in \mathcal{R} \setminus \{r\}$. 
If $\sigma$ satisfies edge-directionality and $r$ is an independent edge w.r.t. $\alpha$, then $\sigma(\alpha)=\sigma_{w'}(\alpha)$ for any $w'(r) \in [0,1]$. 
Thus, by Definition~\ref{def_grae}, we have $\nabla_{r \mapsto \alpha}^{\sigma} \equiv 0$ for any $w'(r) \in [0,1]$, which completes the proof.
\end{proof}

\begin{proposition}
Let $r$ be a direct or indirect edge w.r.t. $\alpha$.
$\nabla_{r \mapsto \alpha}^{\sigma}$ satisfies counterfactuality if $\sigma$ satisfies both edge-monotonicity and monotonicity.
\end{proposition}
\begin{proof}
Let us first consider the \textbf{direct} case.
\begin{itemize}
    \item If $\nabla_{r \mapsto \alpha}^{\sigma} < 0$, then $r$ must be an attack edge. We can prove this by contradiction. Suppose, for contradiction, that $r\in\mathcal{R}^{+}$. Then, by Proposition~\ref{proposition_direct_influence}, we must have $\nabla_{r \mapsto \alpha}^{\sigma} \geq 0$, which contradicts the assumption that $\nabla_{r \mapsto \alpha}^{\sigma} < 0$. Hence, we know $r$ is a direct attack. Moreover, by edge-monotonicity, setting $w(r)$ to its minimum weight $0$ will weaken the attack, resulting in $\sigma(\alpha) \leq \sigma_{w'}(\alpha)$.
    \item Similarly, if $\nabla_{r \mapsto \alpha}^{\sigma} > 0$, $r$ must be a support edge; otherwise, it would contradict Proposition~\ref{proposition_direct_influence}. By edge-monotonicity, setting $w(r)=0$ weakens the support, so $\sigma(\alpha) \geq \sigma_{w'}(\alpha)$.
\end{itemize}
Let us now consider the \textbf{indirect} case. 
\begin{itemize}
    \item If $\nabla_{r \mapsto \alpha}^{\sigma} < 0$, then one of the following must hold: 
    
    1. $r\in\mathcal{R}^{+}$ and $r$ passes through odd number of attacks; or
    
    2. $r\in\mathcal{R}^{-}$ and $r$ passes through even number of attacks.
    
    Suppose, for contradiction, that either (i) $r\in\mathcal{R}^{+}$ and $r$ passes through even number of attacks; or (ii) $r\in\mathcal{R}^{-}$ and $r$ passes through odd number of attacks. In both cases, by Proposition~\ref{proposition_sign_correct_indirect}, we have $\nabla_{r \mapsto \alpha}^{\sigma} \geq 0$, which contradicts the assumption. Thus, only the two listed cases are possible.
    
    In either case, decreasing of $w(r)$ will not increase the strength of its direct child argument (by edge-monotonicity), and due to the sign of influence propagation through the path, this results in a non-decreasing effect on $\sigma(\alpha)$ (by monotonicity). Hence, $\sigma(\alpha) \leq \sigma_{w'}(\alpha)$.

    \item Similarly, if $\nabla_{r \mapsto \alpha}^{\sigma} > 0$, we can prove analogously that $\sigma(\alpha) \geq \sigma_{w'}(\alpha)$.
\end{itemize}
\end{proof}

\begin{proposition}
Let $r$ be a direct or indirect edge with respect to $\alpha$.
$\nabla_{r \mapsto \alpha}^{\sigma}$ satisfies qualitative invariability
if $\sigma$ satisfies both edge-monotonicity and monotonicity.
\end{proposition}
\begin{proof}
Let us first consider the \textbf{direct} case. From the proof of Proposition~\ref{proposition_counterfactuality_direct}, we know that:
\begin{itemize}
    \item $\nabla_{r \mapsto \alpha}^{\sigma} < 0 \Rightarrow r \in \mathcal{R}^{-} \text{ is a direct attack}$;
    \item $\nabla_{r \mapsto \alpha}^{\sigma} > 0 \Rightarrow r \in \mathcal{R}^{+} \text{ is a direct support}$.
\end{itemize}
Moreover, by Proposition~\ref{proposition_direct_influence}, we have:

1. If $r \in \mathcal{R^{+}}$, then $\nabla_{r \mapsto \alpha}^{\sigma} \geq 0$;

2. If $r \in \mathcal{R^{-}}$, then $\nabla_{r \mapsto \alpha}^{\sigma} \leq 0$,\\
where these inequalities hold independently of the specific value of $w(r)$.

Let us now consider the \textbf{indirect} case. 
From the proof of Proposition~\ref{proposition_counterfactuality_direct}, we know that if $\nabla_{r \mapsto \alpha}^{\sigma} < 0$, then one of the following must hold: 
    
    1. $r\in\mathcal{R}^{+}$ and $r$ passes through odd number of attacks; or
    
    2. $r\in\mathcal{R}^{-}$ and $r$ passes through even number of attacks.

    Moreover, by Proposition~\ref{proposition_sign_correct_indirect}, we know $\nabla_{r \mapsto \alpha}^{\sigma} \leq 0$ in either case, where the inequality hold independently of the specific value of $w(r)$.

    Similarly, if $\nabla_{r \mapsto \alpha}^{\sigma} > 0$, we have $\nabla_{r \mapsto \alpha}^{\sigma} \geq 0$ regardless of the specific value of $w(r)$.
    
\end{proof}

\begin{proposition}[Tractability]
Let $|\mathcal{A}|=m$ and $|\mathcal{R}|=n$, then G-RAEs can be generated in linear time $\mathcal{O}(m+n)$ for acyclic EW-QBAF.
\end{proposition}
\begin{proof}
To compute the G-RAE for a given topic argument $\alpha$, we first perform a topological sort of the EW-QBAF. Since the graph is acyclic, this can be done in linear time $\mathcal{O}(m+n)$.
The G-RAE corresponds to the partial derivative of the strength of the topic argument with respect to the weight of the influence edge and can be computed by a backward propagation procedure used in training multilayer perceptrons in linear time~\cite{dlbook2016}. (Instead of inverting the direction from the input neurons to the output neurons, we invert the direction given by the topological ordering).
\end{proof}

\begin{proposition}[G-RAE Approximation Complexity]
Let $|\mathcal{A}|=m$ and $|\mathcal{R}|=n$, then all approximate G-RAEs can be generated in time $\mathcal{O}(n \cdot (m+n))$ for acyclic EW-QBAFs.
\end{proposition}
\begin{proof}
Approximating G-RAEs requires computing the strength values of arguments. In acyclic QBAFs, these values can be computed in linear time $\mathcal{O}(m+n)$~\cite[Proposition 3.1]{Potyka19}. Since strength values must be recomputed for each edge weight perturbation and there are $n$ edges, the overall time complexity is $\mathcal{O}(n \cdot (m+n))$.
\end{proof}

\begin{proposition}[Contestability Algorithm Complexity]
Let $|\mathcal{A}|=m$ and $|\mathcal{R}|=n$, and suppose the maximum number of iterations is $M$. Then the time complexity of Algorithm~\ref{algo_contest} is $\mathcal{O}(M \cdot n \cdot (m+n))$ for acyclic EW-QBAFs.
\end{proposition}
\begin{proof}
The algorithm requires at most $M$ iterations. In each iteration, approximating G-RAEs takes $\mathcal{O}(n \cdot (m+n))$ time as stated in Proposition~\ref{proposition_approx_GRAEs}. After that, the algorithm updates all $n$ edges in $\mathcal{O}(n)$ time and recomputes the strength in $\mathcal{O}(m+n)$ time as shown in the proof of Proposition~\ref{proposition_approx_GRAEs}. Therefore, the overall time complexity of Algorithm~\ref{algo_contest} is $\mathcal{O}(M \cdot n \cdot (m+n))$.
\end{proof}

\section{Evaluations}

\paragraph{Hardware Specifications}
We ran all experiments on a Windows PC (OS: Windows 10 Enterprise, Version: 22H2, 64-bit operating system; Processor: AMD Ryzen 7 PRO 3700 8-Core Processor, 3.59 GHz; Memory: 16.0 GB).

\setcounter{table}{0}

\begin{table*}[th]
\small
\centering
\caption{Validity, median and average runtime (in seconds), and number of attempts (average and maximum) under different semantics and edge sizes.}
\label{tab111}
\begin{tabular}{clrrrrrrrrrr}
\toprule
\textbf{Semantics} & \textbf{\# Edges} & \textbf{10} & \textbf{20} & \textbf{30} & \textbf{40} & \textbf{50} & \textbf{60} & \textbf{70} & \textbf{80} & \textbf{90} & \textbf{100} \\
\midrule

&validity         & 100\% & 100\% & 100\% & 100\% & 100\% & 100\% & 100\% & 100\% & 100\% & 100\% \\
&attempts (avg)   & 1.00 & 1.00 & 1.03 & 1.00 & 1.00 & 1.00 & 1.00 & 1.02 & 1.02 & 1.03 \\
QE&attempts (max)   & 1    & 1    & 3    & 1    & 1    & 1    & 1    & 2    & 3    & 4    \\
&runtime (median) & $<0.01$ & 0.01 & 0.02 & 0.03 & 0.04 & 0.07 & 0.06 & 0.08 & 0.11 & 0.11 \\
&runtime (avg)    & 0.01 & 0.04 & 0.13 & 0.07 & 0.13 & 0.31 & 0.21 & 1.04 & 0.95 & 1.27 \\

\midrule

&validity         & 100\% & 100\% & 100\% & 100\% & 100\% & 100\% & 100\% & 100\% & 100\% & 100\% \\
&attempts (avg)   & 1.00 & 1.00 & 1.00 & 1.00 & 1.00 & 1.00 & 1.00 & 1.00 & 1.00 & 1.00 \\
REB&attempts (max)   & 1    & 1    & 1    & 1    & 1    & 1    & 1    & 1    & 1    & 1    \\
&runtime (median) & 0.01 & 0.01 & 0.01 & 0.02 & 0.01 & 0.05 & 0.05 & 0.07 & 0.07 & 0.03 \\
&runtime (avg)    & 0.01 & 0.01 & 0.02 & 0.03 & 0.03 & 0.06 & 0.07 & 0.08 & 0.11 & 0.12 \\

\midrule

&validity         & 100\% & 100\% & 100\% & 100\% & 100\% & 100\% & 100\% & 100\% & 100\% & 100\% \\
&attempts (avg)   & 1.00 & 1.06 & 1.02 & 1.01 & 1.00 & 1.01 & 1.01 & 1.01 & 1.02 & 1.00 \\
DF-QuAD&attempts (max)   & 1    & 4    & 3    & 2    & 1    & 2    & 2    & 2    & 3    & 1    \\
&runtime (median) & 0.01 & 0.01 & 0.02 & 0.03 & 0.04 & 0.06 & 0.07 & 0.08 & 0.10 & 0.10 \\
&runtime (avg)    & 0.01 & 0.06 & 0.07 & 0.07 & 0.05 & 0.15 & 0.33 & 0.27 & 0.52 & 0.17 \\

\bottomrule
\end{tabular}
\end{table*}

\begin{table*}[!thb]
\centering
\small
\caption{Validity, median and average runtime (in seconds), and number of attempts (average and maximum) under different connection probabilities for varying MLP-like structures.}
\label{tab222}
\begin{tabular}{clrrrrrrrrrr}
\toprule
\textbf{Structure} & \textbf{Connect\_prob} & \textbf{0.1} & \textbf{0.2} & \textbf{0.3} & \textbf{0.4} & \textbf{0.5} & \textbf{0.6} & \textbf{0.7} & \textbf{0.8} & \textbf{0.9} & \textbf{1.0} \\

\midrule

&validity         & 100\% & 100\% & 100\% & 100\% & 100\% & 100\% & 100\% & 100\% & 100\% & 100\% \\
&attempts (avg)   & 1.00 & 1.00 & 1.00 & 1.00 & 1.00 & 1.00 & 1.01 & 1.00 & 1.00 & 1.00 \\
$[8,32,1]$&attempts (max)   & 1    & 1    & 1    & 1    & 1    & 1    & 2    & 1    & 1    & 1    \\
&runtime (median) & 0.02 & 0.06 & 0.12 & 0.17 & 0.28 & 0.41 & 0.53 & 0.63 & 0.77 & 0.85 \\
&runtime (avg)    & 0.02 & 0.08 & 0.17 & 0.31 & 0.45 & 0.58 & 1.30 & 1.17 & 0.84 & 1.61 \\

\midrule

&validity         & 100\% & 100\% & 100\% & 100\% & 100\% & 100\% & 100\% & 100\% & 100\% & 100\% \\
&attempts (avg)   & 1.00 & 1.00 & 1.00 & 1.00 & 1.00 & 1.00 & 1.00 & 1.00 & 1.00 & 1.00 \\
$[8,32,16,1]$&attempts (max)   & 1    & 1    & 1    & 1    & 1    & 1    & 1    & 1    & 1    & 1    \\
&runtime (median) & 0.08 & 0.26 & 0.55 & 1.05 & 1.80 & 2.08 & 2.51 & 3.69 & 4.49 & 5.81 \\
&runtime (avg)    & 0.09 & 0.46 & 1.17 & 1.94 & 2.31 & 4.11 & 3.88 & 4.59 & 5.14 & 6.63 \\

\midrule

&validity         & 100\% & 100\% & 100\% & 100\% & 100\% & 100\% & 100\% & 100\% & 100\% & 100\% \\
&attempts (avg)   & 1.00 & 1.00 & 1.00 & 1.00 & 1.00 & 1.00 & 1.00 & 1.00 & 1.00 & 1.00 \\
$[8,32,16,8,1]$&attempts (max)   & 1    & 1    & 1    & 1    & 1    & 1    & 1    & 1    & 1    & 1    \\
&runtime (median) & 0.01 & 0.27 & 0.48 & 1.06 & 1.79 & 2.50 & 3.50 & 4.58 & 5.98 & 6.44 \\
&runtime (avg)    & 0.06 & 0.39 & 0.92 & 2.53 & 2.23 & 3.95 & 6.97 & 12.50 & 9.41 & 7.95 \\

\bottomrule
\end{tabular}
\end{table*}

\end{document}